\documentclass{article}

\usepackage{microtype}
\usepackage{graphicx}
\usepackage{booktabs} 
\usepackage{paralist}
\usepackage{calligra}
\usepackage[T1]{fontenc}
\usepackage{enumitem}
\usepackage[font={small}]{caption}

\usepackage{sidecap}
\usepackage{wrapfig}

\usepackage{float}
\usepackage{amsmath}
\usepackage{amsthm}
\usepackage{amssymb}
\usepackage{graphicx}
\usepackage{bm}
\usepackage{bbm}

\usepackage{algorithm}
\usepackage{algorithmic}

\usepackage{hyperref}

\usepackage{caption}
\usepackage{subcaption}
\usepackage{color}
\definecolor{airforceblue}{rgb}{0.36, 0.54, 0.66}
\definecolor{cinnamon}{rgb}{0.82, 0.41, 0.12}

\definecolor{orange}{rgb}{1.0, 0.5, 0.0}

\newcommand{\myparagraph}[1]{\smallskip\noindent\textbf{#1}}

\newcommand{\E}{\mathbb{E}}

\newcommand{\R}{\mathbb{R}}
\newcommand{\Y}{\mathbb{Y}}
\newcommand{\iX}{\mathbb{X}}

\newcommand{\bx}{\bm{x}}
\newcommand{\by}{\bm{y}}
\newcommand{\bw}{\bm{w}}
\newcommand{\Wo}{\bw_{\scriptscriptstyle opt}}
\newcommand{\bX}{{\mathbf X}}
\newcommand{\Dd}{\mathcal{D}}
\newcommand{\pR}{1}
\newcommand{\rR}{2}
\newcommand{\prR}{i}
\newcommand{\hdD}{{h}}
\newcommand{\hD}{{h}}
\newcommand{\qhard}{\alpha}
\newcommand{\capB}{\beta}
\newcommand{\ahard}{a}
\newcommand{\qDel}{\Delta}
\newcommand{\psn}{\R_{\ge 0}}
\newcommand{\prn}{\R_{> 0}}
\newcommand{\tpcclustering}{TPC_{DC}}
\newcommand{\tpcrepresent}{TPC_{RP}}

\newcommand{\hlt}[1]{{\color{airforceblue} #1}}

\newcommand{\texq}[1]{\hspace{-1.5em}\mathrm{\hlt{#1}}}

\newcommand{\app}{App.}

\newtheorem{theorem}{Theorem}
\newtheorem{lemma}{Lemma}

\newtheorem{corollary}{Corollary}

\newtheorem{defn}{Definition}

\usepackage[accepted]{icml2022}

\icmltitlerunning{Active Learning on a Budget: Opposite Strategies Suit High and Low Budgets}

\begin{document}

\twocolumn[
\icmltitle{Active Learning on a Budget: Opposite Strategies Suit High and Low Budgets}



\icmlsetsymbol{equal}{*}

\begin{icmlauthorlist}
\icmlauthor{Guy Hacohen}{huji,elsc,equal}
\icmlauthor{Avihu Dekel}{huji,equal}
\icmlauthor{Daphna Weinshall}{huji}
\end{icmlauthorlist}

\icmlaffiliation{huji}{School of Computer Science and Engineering, The Hebrew
University of Jerusalem, Jerusalem, Israel}
\icmlaffiliation{elsc}{Edmond \& Lily Safra Center for Brain Sciences, The Hebrew University
of Jerusalem, Jerusalem, Israel}

\icmlcorrespondingauthor{Guy Hacohen}{guy.hacohen@mail.huji.ac.il}
\icmlcorrespondingauthor{Avihu Dekel}{avihu.dekel@mail.huji.ac.il}
\icmlcorrespondingauthor{Daphna Weinshall}{daphna@cs.huji.ac.il}

\icmlkeywords{Active Learning, Low budget, AL, Semi-supervised learning, SSL, initial pooling, TypiClust, ICML}

\vskip 0.3in
]
 



\printAffiliationsAndNotice{\icmlEqualContribution} 

\begin{abstract}

Investigating active learning, we focus on the relation between the number of labeled examples (budget size), and suitable querying strategies. Our theoretical analysis shows a behavior reminiscent of phase transition: typical examples are best queried when the budget is low, while unrepresentative examples are best queried when the budget is large. Combined evidence shows that a similar phenomenon occurs in common classification models. Accordingly, we propose \emph{TypiClust} -- a  deep active learning strategy suited for low budgets. In a comparative empirical investigation of supervised learning, using a variety of architectures and image datasets, \emph{TypiClust} outperforms all other active learning strategies in the low-budget regime. Using \emph{TypiClust} in the semi-supervised framework, performance gets an even more significant boost. In particular, state-of-the-art semi-supervised methods trained on CIFAR-10 with $10$ labeled examples selected by \emph{TypiClust}, reach $93.2\%$ accuracy -- an improvement of $39.4\%$ over random selection. Code is available at \href{https://github.com/avihu111/TypiClust}{https://github.com/avihu111/TypiClust}.

\end{abstract}

\section{Introduction}

Recent years have witnessed the emergence of deep learning as the dominant force in advancing machine learning and its applications. But this development is data-hungry -- to a large extent, deep learning owes its success to the growing availability of annotated data. This is problematic even in our era of \emph{Big Data}, as the annotation of data remains costly.


\begin{figure}[thb]

\begin{subfigure}{.235\textwidth}
  \centering
 \includegraphics[width=\linewidth,height=2.7cm]{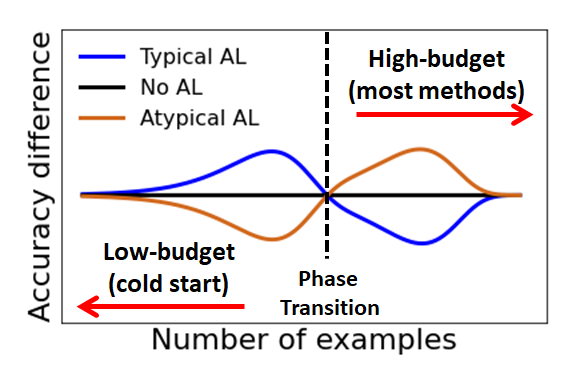}
\caption{Theoretical results}
\label{fig:example}
\end{subfigure}
\begin{subfigure}{.235\textwidth}
  \centering
 \includegraphics[width=\linewidth,height=2.7cm]{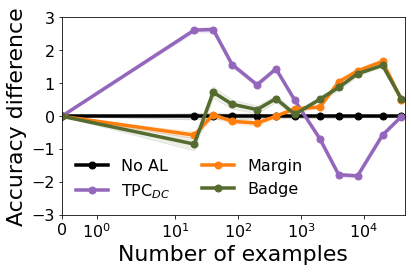}
\caption{Empirical results}
\label{fig:phase_transition_real}
\end{subfigure}

\vspace{-.2cm}
\caption{
Visualization of phase transition in deep active learning, as revealed by plotting the difference in accuracy between different AL strategies to a random baseline as a function of budget size (the number of labeled examples). We see similar behavior both theoretically and empirically: when the budget is low, oversampling typical examples is more beneficial, whereas when the budget is high, oversampling atypical examples is more beneficial. (a) The behavior of an idealized model (see Section~\ref{sec:theoretical_analysis}). (b) The behavior of \textit{TypiClust}, contrasted with 2 basic uncertainty-based strategies, as seen in deep neural models trained on CIFAR-10 (see Section~\ref{sec:emp}).} 
\vspace{-.25cm}
\end{figure}

Active Learning (AL)
aims to alleviate this problem \citep[see surveys in][] {settles.tr09,schroder2020survey}. Given a large pool of unlabeled data, and possibly a small set of labeled examples, learning is done iteratively: the learner employs the labeled examples (if any), then queries an oracle by submitting examples to be annotated. This may be done repeatedly, often until a fixed budget is exhausted.




Many traditional active learning approaches are based on uncertainty sampling \citep[e.g.,][] {lewis1994sequential,ranganathan2017deep,gissin2019discriminative,sinha2019variational}. In uncertainty sampling, the learner queries examples about which it is least certain, presumably because such labels contain the most information about the problem.
Another principle guiding deep AL approaches is diversity sampling
\citep[e.g.,][] {hu2010off,elhamifar2013convex,sener2018active,shui2020deep}.
Here, queries are chosen to minimize the correlation between samples, in order to avoid redundancy in the annotations.
Sensibly, diversity sampling is often combined with uncertainty sampling 
(see \app~\ref{app:related_work} for further discussion of related work).

At present, effective deep active learning strategies are known to require a large initial set of labeled examples to work properly \citep{DBLP:conf/emnlp/YuanLB20,pourahmadi2021simple}. We call this the \emph{high budget regime}. In the \emph{low budget regime}, where the initial labeled set is small or absent, it has been shown that random selection outperforms most deep AL strategies \citep[see][]{attenberg2010label,mittal2019parting,DBLP:journals/tkde/ZhuLHWGLC20,simeoni2021rethinking,chandra2021initial}. This ``cold start'' phenomenon is often explained by the poor ability of neural models to capture uncertainty, which is more severe with a small budget of labels \citep{nguyen2015deep,gal2016dropout}. The low-budget scenario is relevant in many applications, especially those requiring an expert tagger whose time is expensive. If we want to expand deep learning to these domains, overcoming the cold start problem becomes an important challenge.

In this paper, we suggest that the low and high budgets regimes are qualitatively different, and require \emph{opposite} querying strategies. Furthermore, we claim that the uncertainty principle is only suited for the high-budget regime, while the opposite strategy -- the selection of the least ambivalent points -- is suitable for the low-budget regime.

We begin, in Section~\ref{sec:theoretical_analysis}, by establishing the theoretical foundations for this claim. We analyze a mixture model where two general learners, each limited to a distinct region of the input space, are independently learned. In this framework, we see a phase-transition-like phenomenon: in the low-budget regime, over-sampling the ``easier'' region, which can be learned from fewer examples, improves the outcome of learning. In the high-budget regime, over-sampling the alternative region is more beneficial. In other words, opposing querying strategies are suitable for the low-budget and high-budget regimes. This is illustrated in Fig.~\ref{fig:example}.

We continue by identifying a set of sufficient conditions, which guarantee that two independent learners display this phase-transition-like phenomenon. We then give a formal argument, showing that linear classifiers satisfy these conditions. We further provide empirical evidence to the effect that neural models may also satisfy these conditions. 

\begin{figure}[b!]
\begin{center}
\vspace{-0.3cm}
\includegraphics[width=.48\textwidth]{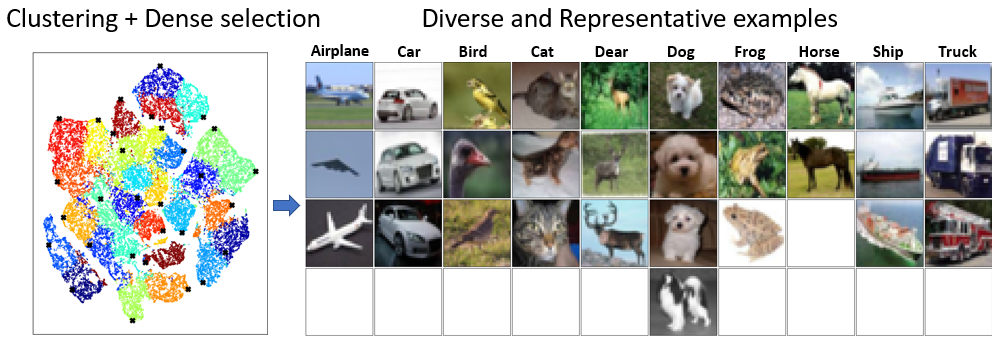} \\
\footnotesize{(a) \hspace{3.5cm} (b) \hspace{1.5cm}}
\vspace{-0.2cm}
\caption{Visualizing the selection of 30 examples from CIFAR-10 by \emph{TypiClust}. (a) The data is first clustered into 30 clusters, and the densest region within every cluster is sampled. We show t-SNE dimensionality reduction of the feature space, colored by cluster assignment, where selected examples are marked by $\times$. (b) The selected images, organized column-wise by class. Note that the ensuing labeled set is approximately class-balanced, even though the queries are chosen without access to class labels.}
\label{fig:cifar10_examples}
\end{center}
\end{figure}

Previous art established that in the high-budget regime, it is beneficial to preferentially sample uncertain examples. The phase transition result predicts that in the low-budget regime, a strategy that preferentially samples the most certain examples is beneficial. However, estimating prediction certainty is difficult, and cannot be reliably accomplished in the low-budget regime with access to very few labeled examples. We therefore adopt an alternative approach, replacing the notion of \emph{certainty} with the notion of \emph{typicality} (loosely defined): a point is typical if it lies in a high-density region of the input space, irrespective of labels. 

In Section~\ref{sec:active_learning_low_budget}, guided by these observations, we propose Typical Clustering (\emph{TypiClust}) -- a strategy for active learning in the low-budget regime. \emph{TypiClust} aims to pick a diverse set of typical examples, which are likely to be representative of the entire dataset. To this end, \emph{TypiClust} employs self-supervised representation learning and then estimates each point's density in this representation. Diversity is obtained by clustering the dataset and sampling the densest point from each cluster (see Fig.~\ref{fig:cifar10_examples}). 

In Section~\ref{sec:emp}, we compare \emph{TypiClust} to various AL strategies in the low-budget regime. \emph{TypiClust} consistently improves generalization by a large margin, across different datasets and architectures, reaching state-of-the-art (SOTA) results in many problems. In agreement with \citet{DBLP:journals/tkde/ZhuLHWGLC20}, we also observe that the alternative AL strategies are not effective in this domain, and are even detrimental.

\emph{TypiClust} is especially beneficial for semi-supervised learning. Although in the AL framework the learner has access to a big pool of unlabeled data by construction, most AL strategies do not exploit the unlabeled data for learning, beyond query selection. Recent studies report that the benefit of AL is marginal when incorporated into semi-supervised learning \citep{chan2021marginal,bengar2021reducing}, with little added value over the random selection of labels. Re-examining this observation, we note that semi-supervised learning is most beneficial in the low-budget regime, wherein the explored AL strategies are inherently not suitable. When incorporating \emph{TypiClust}, which is designed for the low-budget regime, into semi-supervised learning, performance is indeed improved by a large margin. 

\vspace{-.2cm}
\subsection*{Summary of Contribution}

\setlist{nolistsep}
\setlength\itemsep{30em}
\begin{enumerate}[leftmargin=0.5cm,label=\roman*.]

\setlength\itemsep{.3em}
    \item A novel theoretical model analyzing Active Learning (AL) as biased sampling strategies in a mixture model.
    \item Prediction of the cold start phenomenon in AL.
    \item Prediction that opposite strategies suit AL in the low-budget and high-budget regimes.
    \item Empirical support of these theoretical principles.
    \item \emph{TypiClust}, a novel strategy that significantly improves active learning in the low-budget regime.
    \item Large performance boost to SOTA semi-supervised methods by \emph{TypiClust}.

\end{enumerate}

\section{Theoretical Analysis}
\label{sec:theoretical_analysis}

Given a large pool of $U$ unlabeled examples and a possibly small (or even empty) set of $L$ labeled examples, an Active Learning (AL) method selects a small subset of $B$ examples from $U$ to submit as label queries to an oracle. We call the number of labeled examples known to the learner \emph{budget}, whose size is $m=B+L$. In this section, we aim to study the optimal query selection strategy as it depends on $m$.

To this end, we analyze a mixture model of two general learners. In \S\ref{sec:model}, the model is defined by first splitting the support of the data distribution into two distinct regions, $R_1$ and $R_2$, further assuming that each region is independently learned by its own general learner. $R_1$ and $R_2$ are distinguished by the property that if they are learned independently, $R_\pR$ is \emph{easier to learn} than $R_\rR$ (as formalized in Def.~\ref{def:ease} below). We then define the error score $E_\Dd(m)$ of this model, which measures the expected error of the model over all training samples of size $m$, as a function of $m$. 

Within this framework, in \S\ref{sec:biased} we derive a \textbf{threshold test} on the budget size $m$ and $E_\Dd(m)$, which determines whether an optimal AL strategy should oversample $R_1$  or $R_2$. In \S\ref{sec:error-fs} we obtain sufficient conditions on $E_\Dd(m)$, which guarantee phase transition as illustrated in Fig.~\ref{fig:example}. In accordance, an optimal AL strategy will oversample $R_1$ if $m$ is smaller than some threshold $m_0$, and oversample $R_2$ otherwise. We let the term \emph{low budget regime} denote budgets with $m\leq m_0$, and \emph{high budget regime} denote budgets with $m>m_0$. We may now conclude that for any learner model whose error score meets our sufficient conditions, opposite strategies suit the low and high budget regimes. 

In \S\ref{sec:common} we analytically prove that the error score of a mixture of two linear classifiers satisfies the required conditions, while in \app~\ref{sec:error-dnn} we empirically show that this is the case also with deep neural networks.

\subsection{Mixture Model Definition}
\label{sec:model}

We analyze a mixture of two learners model, where each learner is independently trained on a distinct part of the domain of the data distribution. Formally, let $\bX=[\bx,y]$ denote a single data point, where $\bx\in\R^d$ denotes an input point and $y\in\Y$ its corresponding label or target value. For example, $\Y$ is $\R$ in regression problems, and $[k]$ in classification problems. Each point $\bX$ is drawn from a distribution $\Dd$ with density $f_\Dd(\bX)$. We denote an i.i.d. sample of $m$ points from $\Dd$ as $\iX^m=\left\{\bX_1,...,\bX_m\right\}\sim\Dd^m$.

Let $R_\pR,R_\rR\subseteq\mathbb{R}^{d}\times\Y$ denote a partition of the domain of $f_\Dd$, where $R_\pR\cup R_\rR=\mathbb{R}^{d}\times\Y$ and $R_\pR\cap R_\rR=\emptyset$. 
Let $\Dd_\pR, \Dd_\rR$ denote the conditional distributions obtained when restricting $\Dd$ to regions $R_\pR, R_\rR$ respectively. Note that $\Dd$ can now be viewed as a mixture distribution, where points are sampled from $\Dd_\pR$ with probability  $p=\int_{\bX\in R_\pR}f_\Dd(\bX)d\bX$, and $\Dd_\rR$ with probability $(1-p)$.
Let $m_\prR,~\prR\in[2]$ denote the number of points in $R_\prR$ when sampling $m$ points from $\Dd$, and $\iX^{m_\prR}$ denote  the restriction of sample $\iX^m$ to $R_\prR$. We denote the hypothesis of a learner when trained independently on $\iX^{m_\prR}$ as $\hdD(\iX^{m_\prR})$.

Next, we define the error score of a learner, which is a function of $m$ -- the training set size. It measures the expected generalization error of the learner over all such training sets.
\begin{defn}[Error score] 
Assume training sample $\iX^m\sim\Dd^m$, with the corresponding learned hypothesis $\hdD(\iX^{m})$ -- a random variable whose distribution is denoted $\Dd_h$. Let $Er\left(h\left(\iX^{m}\right)\right)$ denote the expected generalization error of this hypothesis. The expected error of the learner, over all training sets of size $m$, is given by
\begin{equation*}
E_{\Dd}(m) = \E_{\iX^{m}\sim\Dd^m}\left[\E_{\hD(\iX^{m})\sim\Dd_h}\left[Er\left(h\left(\iX^{m}\right)\right)\right]\right].
\end{equation*}
\label{def:expected-error}
\vspace{-.5cm}
\end{defn} 

We adopt two common assumptions regarding $E_{\Dd}(m)$. \begin{inparaenum}[(i)] \item  \textbf{Efficiency}: $E_{\Dd}(m)$ is strictly monotonically decreasing, namely, on average the learner benefits from additional examples. \item \textbf{Realizability}: $\lim\limits_{m \rightarrow\infty}E_{\Dd}(m) =0$. \end{inparaenum}

During training, we assume a mixture of independent learners in $R_\pR,R_\rR$, and a training set composed of $m_\pR,m_\rR$ examples from each region respectively. The error score of the mixture learner on $\Dd$ for $m=m_\rR+m_\pR$ is:
\begin{equation}
\begin{split}
\label{eq:optimization}
    E_{\Dd}(m) = p\cdot E_{\Dd_\pR}(m_\pR) + (1-p)\cdot E_{\Dd_\rR}(m_\rR).
\end{split}
\end{equation}
As an important ingredient of the mixture model, we assume that one region requires fewer examples to be adequately learned, and call this region $R_\pR$. Essentially, we expect the error score to decrease faster at $R_\pR$, with $E_{\Dd_\pR}(m)<E_{\Dd_\rR}(m)~\forall m$. Other than this difference, we expect the error score to be similar in $R_\pR$ and $R_\rR$. 

Making this notion more precise, we define an order relation on partition $R_\pR, R_\rR$ as follows:
\begin{defn}[Order $R_\pR\prec R_\rR$] 
Let $R_\pR, R_\rR$ denote a partition of the domain of $f_\Dd$. Assume that the error score in $R_\pR$ and $R_\rR$ can be written as $E_{\Dd_\pR}(m) =E(m)$ and $E_{\Dd_\rR}(m)= E(\qhard m)$ for a single function $E(m)$ and $\qhard>0$. We say that $R_\pR$ is easier to learn than $R_\rR$ and denote {$R_\pR\prec R_\rR$} if $\qhard<\frac{p}{1-p}$, where $p$ is the probability of $R_\pR$.
\label{def:ease}
\end{defn} 
Note that $\qhard < 1$ if $p=0.5$. We now assume that $R_\pR\prec R_\rR$, and rewrite (\ref{eq:optimization}) as follows:
\begin{equation}
\begin{split}
\label{eq:Etotal}
E_\Dd(m)&= p \cdot E(m_\pR) + (1-p) \cdot E(\qhard m_\rR).
\end{split}
\end{equation}
Finally, we extend $E(m)$ to domain $\psn$ with the continuation of $E(m)$ denoted $E(x):\psn\rightarrow\psn$, which is in $C^\infty$ 
and positive. \textbf{Efficiency} is extended to imply $E'(x)<0$. 
\subsection{Deriving the Optimal Sampling Strategy}
\label{sec:biased}

Considering the extended error score $E(x):\psn\rightarrow\psn$, we define a biased sampling strategy as follows:
\begin{equation*}
\label{eq:m}
m_\pR =p\cdot m+\qDel, \qquad m_\rR =\left(1-p\right)\cdot m-\qDel.
\end{equation*}

$\qDel=0$ is essentially equivalent to random sampling from $\Dd$.
$\qDel>0$ implies that more training points are sampled from $R_\pR$ than $R_\rR$, and vice versa.

In Thm.~\ref{thm:threshold_test} we show that to minimize the expected generalization error while considering a mixture model of two independent learners as defined above, choosing between the different sampling strategies can be done using a simple threshold test. 

\begin{theorem}
\label{thm:threshold_test}
Given partition $R_\pR\prec R_\rR$ and error score $E(x)$, let $p=Prob(R_\pR)$ and $0<\qhard<\frac{p}{1-p}$. The following threshold test decreases the error score for sample size $m$:
\begin{equation*}
\boxed{
\frac{E'\left(pm\right)}{E'\left(\qhard\left(1-p\right)m\right)} 
\begin{cases} > \frac{\qhard\left(1-p\right)}{p}&\implies \begin{array}{cc}&\texq{over~sample}\vspace{-.15cm}\\  &\texq{region~R_\pR} \end{array} \\
< \frac{\qhard\left(1-p\right)}{p}&\implies \begin{array}{cc}&\texq{over~sample}\vspace{-.15cm}\\  &\texq{region~R_\rR} \end{array}
\end{cases}
}
\end{equation*}
\end{theorem}

\begin{proof}
Starting from (\ref{eq:Etotal}), we obtain the test whereby the error score $E_\Dd(m)$ decreases when $\qDel>0$
\begin{equation*}
\begin{split}
&p\cdot E\left(p\cdot m+\qDel\right)+\left(1-p\right)\cdot E\left(\qhard\left(\left(1-p\right)m-\qDel\right)\right)  \\
&\quad<p\cdot E\left(p\cdot m\right)+\left(1-p\right)\cdot E\left(\qhard\left(1-p\right)m\right)\\
\implies  & \left(1-p\right)\left[E\left(\qhard\left(1-p\right)m\right)-E\left(\qhard\left(\left(1-p\right)m-\qDel\right)\right)\right]\\
&\quad>p\left[E\left(p\cdot m+\qDel\right)-E\left(p\cdot m\right)\right]\\
\implies  &\qhard\left(1-p\right)\frac{E\left(\qhard\left(1-p\right)m\right)-E\left(\qhard\left(1-p\right)m-\qhard\qDel \right)}{\qhard \qDel}\\
&\quad> p\frac{E\left(p\cdot m+\qDel\right)-E\left(p\cdot m\right)}{\qDel}.\\
\end{split}
\end{equation*}
Since $E(x)$ is differentiable and strictly monotonically decreasing with $E'<0$, in the limit of infinitesimal $\qDel$ 
\begin{equation*}
\frac{E'\left(pm\right)}{E'\left(\qhard\left(1-p\right)m\right)}>\frac{\qhard\left(1-p\right)}{p}.
\end{equation*}
The proof for $\qDel<0$ is similar.
\end{proof}

\myparagraph{Example: exponentially decreasing function.}
Assume $E\left(m\right)=e^{-m}$, and a mixture model with $p=0.8,\qhard=0.1$. We simulate the error in (\ref{eq:Etotal}) when biasing the train sample with $\qDel=\pm 0.01$. 
Fig.~\ref{fig:example} shows the differences between the error score of biased sampling (in favor of either $R_\pR$ in blue or $R_\rR$ in orange) and random sampling, as a function of the number of examples $m$. For small $m$ it is beneficial to favorably bias region $R_\pR$, while for large $m$ it is beneficial to favorably bias $R_\rR$. This is the behavior often seen in our empirical investigation, see Fig.~\ref{fig:phase_transition_real} and Section~\ref{sec:emp}. 

\subsection{Error Scores Analyzed}
\label{sec:error-fs}

We now report sufficient conditions on $E(m)$, which guarantee the phase-transition-like behavior illustrated in Fig.~\ref{fig:example}, starting with some formal definitions.

Given partition $R_\pR\prec R_\rR$ and Error score $E(x)$, we say that $E(x)$ is undulating if it displays the following behavior: in the beginning, when the number of training examples is small, the generalization error decreases faster when over-sampling region $R_\pR$. In the end, after seeing sufficiently many training examples, the generalization error decreases faster when  over-sampling region $R_\rR$. Formally: 
\begin{defn}[Undulating] An error score $E(m)$ is undulating if there exist $z_1,z_2\in\R$ such that  $\frac{E'\left(pm\right)}{E'\left(\qhard\left(1-p\right)m\right)}>\frac{\qhard\left(1-p\right)}{p}~~\forall m< z_1$, and $\frac{E'\left(pm\right)}{E'\left(\qhard\left(1-p\right)m\right)}<\frac{\qhard\left(1-p\right)}{p}~~\forall m> z_2$.
\label{def:1}
\end{defn}
For undulating error scores, there could potentially be any number of transitions between the two conditions, switching the preference of $R_\pR$ to $R_\rR$, and vice versa. We extend the above definition to capture a case of particular interest, where this transition occurs only once, as follows: 
\begin{defn}[SP-undulating] An error score $E(m)$ is Single-Phase undulating if it is undulating, and $z_1=z_2$.
\label{def:2}
\end{defn} 

\subsubsection{Undulating Error Scores}
As motivated in Section~\ref{sec:model}, we define a proper error score as follows:
\begin{defn}[Proper error score] $E(x):\psn\rightarrow\psn$ is a proper error score if it is a positive twice differentiable function, which is strictly monotonically decreasing ($E>0$, $E'<0$), $E(0)=c_0\in\prn$, and where $\lim\limits_{x\rightarrow\infty}E(x)=0$.
\label{def:proper}
\end{defn} 

Not all proper error scores will exhibit the phase undulating behavior. In Thm.~\ref{thm:suffiecnt_coniditions_undulating} we state sufficient conditions that ensure this behavior (see proof in \app~\ref{app:proof_for_thm_2}).
\begin{theorem}[Undulating error score: sufficient conditions]
\label{thm:suffiecnt_coniditions_undulating}
Given partition $R_\pR\prec R_\rR$ and Error score $E(x)$, let $p=Prob(R_\pR)$ and $0<\qhard<\frac{p}{1-p}$. 
$E(x)$ is undulating if the following assumptions hold:
\begin{enumerate}[label=(\roman*)]
\item $E(x)$ is a proper \emph{error score} (see Def.~\ref{def:proper}).
\label{T2a1}
\item 
\label{T2a2}
$\lim\limits_{x\to\infty}\frac{E'(x)}{E(x)},\lim\limits_{x\rightarrow\infty}\frac{E(x)}{E(\ahard x)},\lim\limits_{x\rightarrow\infty}\frac{E'(x)}{E'(\ahard x)}$ exist $\forall\ahard\!\in\! (0,1)$.
\item 
\label{T2a3}
$-\log(E(x))\in\omega\left(\log(x)\right)$. 
\end{enumerate}
\end{theorem}



\begin{corollary}[Exponential error as a bound]
\label{cor:exp_undulating}
\emph{Error score} $E(x)$ is undulating if it satisfies assumptions \ref{T2a1} and \ref{T2a2} of Thm.~\ref{thm:suffiecnt_coniditions_undulating}, and is bounded from above as follows
\begin{equation}
\label{eq:bound}
E(x) \leq k e^{-\nu x} \qquad\forall x\in\psn,
\end{equation}
for some constants $\nu,k\in\prn$.
\end{corollary}
\begin{proof}
It can be readily verified that assumption \ref{T2a3} of Thm.~\ref{thm:suffiecnt_coniditions_undulating} follows from (\ref{eq:bound}).
\end{proof}

\subsubsection{SP-Undulating Error Scores}
\label{sec:SP_undulating}
Thm.~\ref{thm:suffiecnt_coniditions_sp_undulating} extends the results of the previous section, by stating a set of sufficient conditions that ensure an SP-undulating error score.
The proof can be found in \app~\ref{app:proof_thm_3}.

\begin{theorem}[SP-undulating: sufficient conditions]
\label{thm:suffiecnt_coniditions_sp_undulating}
Given partition $R_\pR\prec R_\rR$ and Error score $E(x)$, let $p=Prob(R_\pR)$ and $0<\qhard<\frac{p}{1-p}$.  $E(x)$ is SP-undulating if the following assumptions hold:
\begin{enumerate}
\item $E(x)$ is an undulating \emph{error score}.
\item At least one of the following conditions holds:
\begin{enumerate}
\item 
$\frac{-E''\left(x\right)\cdot x}{E'\left(x\right)}$ is
monotonically increasing with $x$.
\item 
$-E'(x)$ is strictly monotonically decreasing and log-concave. 
\end{enumerate}

\end{enumerate}
\end{theorem}

\begin{corollary}[Exponential error is SP-undulating]
\label{cor:exp}
Consider \emph{error scores} of the form $E(x) = k e^{-\nu x}$
for constants $\nu,k\in\prn$. Such functions are SP-undulating. 
\end{corollary}
Cor.~\ref{cor:exp} shows that classifiers with an exponentially decreasing \emph{error score} are SP-undulating, with a single transition from favoring $R_\pR$ to favoring $R_\rR$.


In practice, we cannot assume that the \emph{error score} of commonly used learners is exponentially decreasing. However, 
frequently we can bound the \emph{error score} from above by an exponentially decreasing function, as we demonstrate theoretically (see Section~\ref{sec:piecewise_linear_separator}) and empirically (see Fig.~\ref{fig:error_scores_neural_networks} in \app~\ref{sec:error-dnn}). In such cases, it follows from Cor.~\ref{cor:exp_undulating} that these functions are undulating.  


\subsection{Simple Classification Models}
\label{sec:common}

To make the analysis above more concrete, we analyze a mixture model of two linear classifiers in Section~\ref{sec:piecewise_linear_separator}, as an example of an actual undulating model in common use. Additionally, we analyze the nearest neighbors classification model in Section~\ref{sec:high_density}, to shed light on the rationale behind the partition of the support to regions $R_\pR$ and $R_\rR$. This case analysis demonstrates specific circumstances that make it possible to learn from fewer examples in certain regions.

\subsubsection{Mixture of Two Linear Classifiers}
\label{sec:piecewise_linear_separator}

Consider a binary classification problem and assume a learner that delivers a mixture of two linear classifiers in $\R^d$. The two classifiers are obtained by independently minimizing the $L_2$ loss on points in $R_\pR$ and $R_\rR$ respectively. 

\myparagraph{(i) Bounding the error of each mixture component.}
We first derive a bound on the error separately for $R_\pR$ and $R_\rR$, as it depends on the sample size $m_j$ for $j\in[2]$. Let $X\in\R^{d\times m_j}$ denote the matrix whose columns are the training vectors that lie in region $R_j$. Let $\by\in\{-1,1\}^{m_j}$ denote a row labels vector, where 1 marks positive examples and $-1$ negative examples. The learner seeks a separating row vector $\hat\bw\in\R^d$, where
\begin{equation*}
\hat {\bw} = \mathrm{arg}\min_{{\bw}\in\R^d} \left\|{\bw}X-{\by}\right\|^2 ~~ \implies \hat {\bw} = {\by}X^\top (XX^\top)^{-1}.
\end{equation*}

In Thm.~\ref{thm:error_bound_pw_linear}, we bound the error of a linear model by some exponential function of the number of training examples $m_j$. The proof and further details can be found in \app~\ref{app:proof_thm_4}.

\begin{theorem}[Error bound on a linear classifier]
\label{thm:error_bound_pw_linear}
Assume: \begin{inparaenum}[(i)] \item a bounded sample $\|\bx_i\|\le \beta$, where $XX^\top$ is sufficiently far from singular so that its smallest eigenvalue is bounded from below by $\frac{1}{\Lambda}$; \item a realizable binary problem where the classes are separable by margin $\delta$; \item full rank data covariance, where $\frac{1}{\lambda}$ denotes its smallest singular value. \end{inparaenum} Then there exist some positive constants $k, \nu > 0$, such that $\forall m_j\in\mathbb{N}$ and every sample $\iX^{m_j}$, the expected error of $\hat {\bw}$ obtained using $\iX^{m_j}$ is bounded by:
\begin{equation*}
\E_{\bX\sim\Dd_j}[\mathrm{0-1~loss~of~}\hat {\bw}] \leq ke^{-\nu {m_j}}.
\end{equation*}
\end{theorem}

\myparagraph{(ii) A mixture classifier.}
Assume a mixture of two linear classifiers, and let $E(m)=p\cdot E_{\Dd_\pR}(m_\pR) + (1-p)\cdot E_{\Dd_\rR}(m_\rR)$ denote its error score. The following theorem characterizes this function (the proof can be found in \app~\ref{app:proof_of_thm_5}):
\begin{theorem}[Undulating error]
\label{thm:undulating_pw_linear}
Retain the assumptions of Thm.~\ref{thm:error_bound_pw_linear}, and assume that $\forall\ahard\!\in\! (0,1)$ the following limits exist  $\lim\limits_{m\to\infty}\frac{E'(m)}{E(m)},\lim\limits_{m\rightarrow\infty}\frac{E(m)}{E(\ahard m)}, \lim\limits_{m\rightarrow\infty}\frac{E'(m)}{E'(\ahard m)}$. 
Then the error score of a mixture of two linear classifiers is undulating. 
\end{theorem}

In practice, the \emph{error score} in this case is also SP-undulating, as demonstrated in Fig.~\ref{fig:error_linear_classifier_mixture} in \app~\ref{sec:error-dnn}.

\begin{figure*}[thb!]
    \begin{subfigure}{.33\textwidth}
      \centering
      \includegraphics[width=1\linewidth]{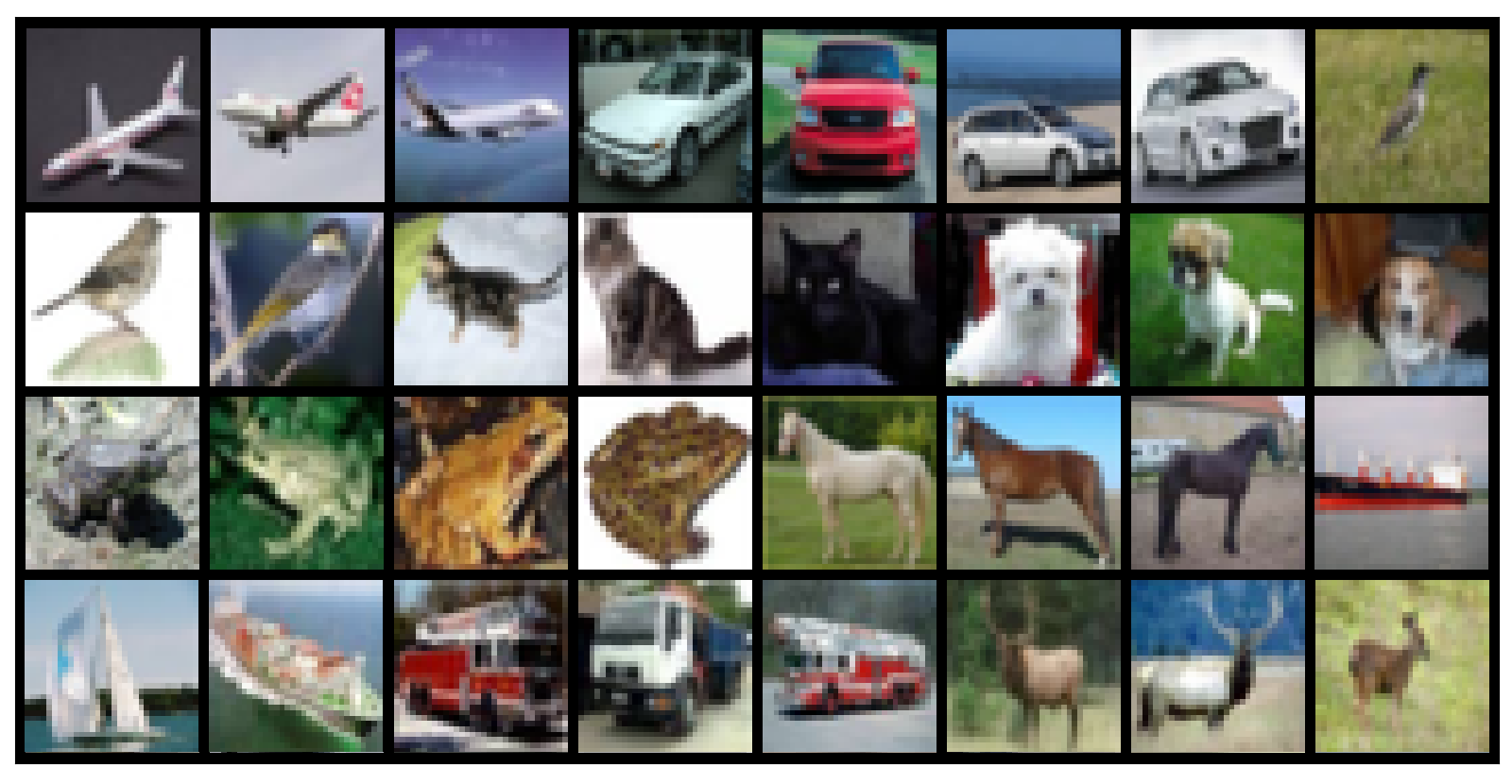}
      \caption{\emph{TypiClust} -- typical and diverse}
      \label{subfig:clustering_matters_regular}
    \end{subfigure}
    \begin{subfigure}{.33\textwidth}
      \centering
      \includegraphics[width=1\linewidth]{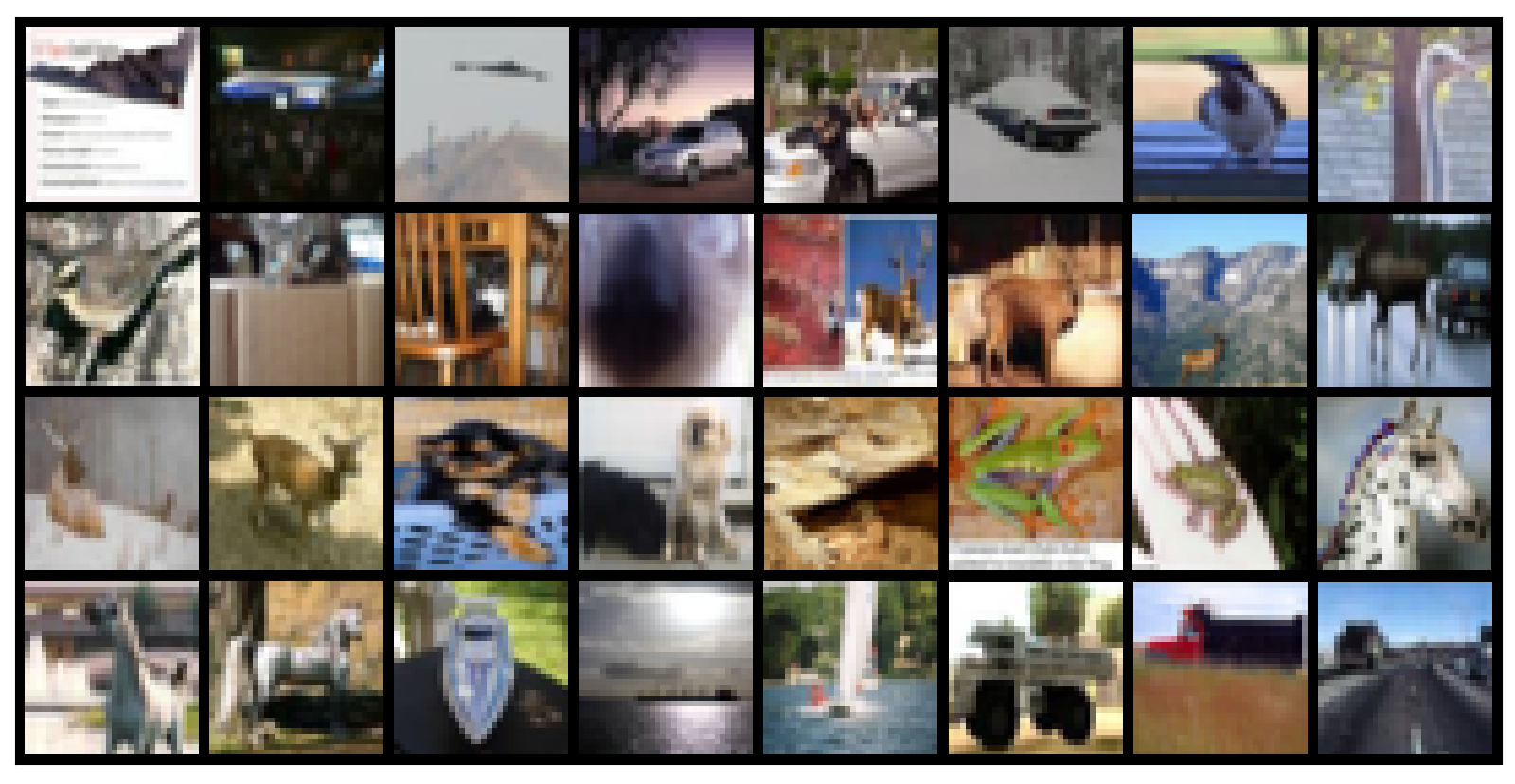}
      \caption{A-typical and diverse}
      \label{subfig:clustering_matters_inverse}
    \end{subfigure}
    \begin{subfigure}{.33\textwidth}
      \centering
      \includegraphics[width=1\linewidth]{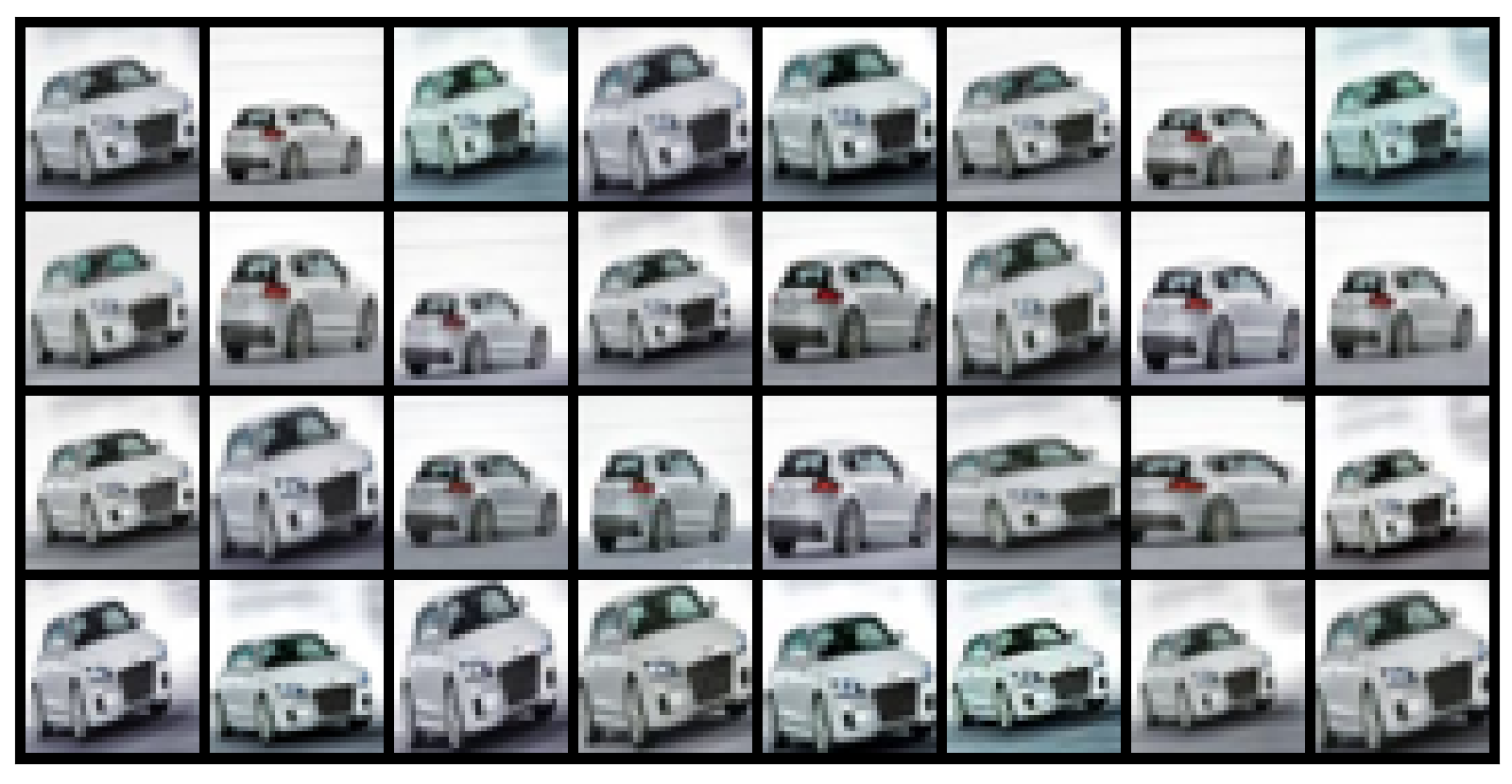}
      \caption{Typical and not diverse}
      \label{subfig:clustering_matters_nocluster}
    \end{subfigure}
    \caption{Qualitative visualization of diversity and typicality in the low budget regime on CIFAR-10. (a) Diverse typical images chosen by \emph{TypiClust}. (b) Picking the \emph{least} typical example in each cluster. (c) Picking the most typical examples, without enforcing diversity.}
    \label{fig:clustering_matters}
\vspace{-0.3cm}
\end{figure*}

\subsubsection{KNN Classifier and High-Density Regions}
\label{sec:high_density}

Our analysis in Section~\ref{sec:error-fs} shows that given some partition of the data into $R_\pR$ and $R_\rR$, where $R_\pR\prec R_\rR$ (see Def.~\ref{def:ease}), then oversampling from $R_\pR$ is preferable in the low budget regime, while oversampling from $R_\rR$ is preferable in the high budget regime. To shed light on the nature of the assumed partition, we analyze below the discrete one-Nearest-Neighbor (1-NN) classification framework. Specifically, we show that \textbf{selecting $R_\pR$ as the set of the most probable points in the dataset} has the property that $R_\pR\prec \{\Omega\setminus{R_\pR}\}$. 

We further show in \app~\ref{app:sec:1_nn_classifier_full_sec} that in this framework, the selection of an initial pool of size $m$ will benefit from the following heuristic: 
\setlist{nolistsep}
\begin{itemize}[noitemsep]
    \item \textbf{Max density:} when selecting a point $\bX_i$, maximize its density $f_\Dd(\bX_i)$.
    \item \textbf{Diversity:} select points that are far apart, so that their corresponding sets of nearest neighbors do not overlap.
\end{itemize}

While 1-NN is a rather simplistic model, we propose to use the derived heuristics to guide deep active learning. In the rest of this paper, we show how these guiding principles benefit deep active learning in the low-budget regime.

\section{Method: Low Budget Active Learning}
\label{sec:active_learning_low_budget}

In the low-budget regime, our theoretical analysis shows that it may be beneficial to bias the training sample in favor of certain regions in the data domain. It also establishes a connection between such regions and the principles of \emph{max density} (or \emph{typicality}) and \emph{diversity}. Here, we incorporate these principles into a simple new active learning strategy called \emph{TypiClust}, designed for the low-budget regime.

\subsection{Framework and Definitions}

Let $L_0$ denote an initial labeled set of examples, and $U_0$ denote an initial unlabeled pool. Active learning is done iteratively: at each iteration $i$, a set of $B$ unlabeled examples is picked according to some strategy. These examples are annotated by an oracle, added to $L_{i-1}$, and removed from $U_{i-1}$. This process is repeated until the labels budget is exhausted, or some predefined termination conditions are satisfied. In the low-budget regime, the total number of labeled examples $\vert L_{i-1}\vert+B$ is assumed to be small. The case where $L_0=\emptyset$ is called ``initial pool selection''.


To capture the principle of max density, we define the \emph{Typicality} of an example by its density in some semantically meaningful feature space. Formally, we measure an example's \emph{Typicality} by the inverse of the average Euclidean distance to its $K$ nearest neighbors\footnote{We use $K=20$, but other choices yield similar results.}, namely: 
\begingroup\abovedisplayskip=4pt\belowdisplayskip=4pt
\begin{equation}
\label{eq:density}
    Typicality(x)=\bigg(\frac{1}{K}\sum_{x_i\in K\text{-NN}(x)}||x-x_i||_2\bigg)^{-1}.
\end{equation}
\endgroup

\subsection{Proposed Strategy: Typical Clustering (\emph{TypiClust})}

In the low-budget regime, an active learning strategy based on typical examples needs to overcome several obstacles: 
\begin{inparaenum}[(a)]
    \item Networks trained on only a few examples are prone to overfit, making measures of typicality noisy and unreliable.
    \item Typical examples tend to be very similar, amplifying the need for diversity.
\end{inparaenum}
The importance of typicality and diversity is visualized in Fig.~\ref{fig:clustering_matters}.

To overcome these obstacles we propose a novel method, called \emph{TypiClust}, which attempts to select typical examples while probing different regions of the data distribution. In our method, self-supervised representation learning is used to overcome (a), while clustering is used to overcome (b). \emph{TypiClust} is therefore composed of three steps:



\myparagraph{Step 1: Representation learning.}
Utilize the large unlabeled pool $U_0$ to learn a semantically meaningful feature space: first train a deep self-supervised task on $U_0 \cup L_0$, then use the penultimate layer of the resulting model as feature space. Such methods are commonly used for semantic feature extraction \citep{chen2020simple, DBLP:conf/nips/GrillSATRBDPGAP20}.

\myparagraph{Step 2: Clustering for diversity.}
As \emph{typicality} in (\ref{eq:density}) is evaluated by measuring distances to neighboring points, the most typical examples are usually close to each other, often resembling the same image (see Fig.~\ref{subfig:clustering_matters_nocluster}). To enforce diversity and thus better represent the underlying data distribution, we employ clustering. Specifically, at each AL iteration $i$, we partition the data into $|L_{i-1}|+B$ clusters. This choice guarantees that there are at least $B$ clusters that do not intersect with the existing labeled examples. We refer to such clusters as \emph{uncovered clusters}.

\myparagraph{Step 3: Querying typical examples.}
We select the most typical examples from the $B$ largest \emph{uncovered clusters}. Selecting from \emph{uncovered clusters} enforces diversity (also w.r.t $L_{i-1}$), while selecting the most typical example in each cluster favors the selection of representative examples.

As the steps above do not depend on any specific representation or clustering method, different variants of the \emph{TypiClust} strategy can be constructed. 
Below we evaluate two variants, both of which outperform by a large margin the uncertainty-based strategies in the low-budget regime:
\begin{enumerate}
    \item \textit{$\tpcclustering$}: Using a deep clustering algorithm both for the self-supervised and clustering tasks. In our experiments, we used SCAN \citep{van2020scan}.
    \item \textit{$\tpcrepresent$}: Using representation learning followed by a clustering algorithm. We used DINO \citep{caron2021emerging} for ImageNet, and SimCLR \citep{chen2020simple} for all other datasets, followed by K-means.
\end{enumerate}
The pseudo-code of \emph{TypiClust} for initial pool selection is given in Alg.~\ref{alg:initial_pool_selection} (see more details in \app~\ref{app:method_implementation_details}). Note that, unlike traditional active learning strategies, \emph{TypiClust} relies on self-supervised representation learning, and therefore can be used for initial pool selection.

\begin{algorithm}[hbt!]
\begin{algorithmic}
\caption{\emph{TypiClust} initial pooling algorithm}
\label{alg:initial_pool_selection}
    \STATE {\bfseries Input:} Unlabeled pool $U$, Budget $B$
    \STATE {\bfseries Output:} $B$ typical and diverse examples to query
    \STATE Embedding $\leftarrow$ Representation\_Learning($U$)
    \STATE Clust $\leftarrow$ Clustering\_algorithm(Embedding,  $B$)
    \STATE Queries $\leftarrow\emptyset$\;
    \FORALL{$i=1,...,B$}
    \STATE Add $\arg\max_{x\in\text{Clust}[i]} \{Typicality(x)\}$ to Queries
    \ENDFOR
    \STATE {\bfseries return} Queries
\end{algorithmic}
\end{algorithm}

\section{Empirical Study}
\label{sec:emp}

We now report our empirical results. Section~\ref{sec:methodology} describes the evaluation protocol, datasets, and baseline methods. Section~\ref{sec:emp_results} describes the actual experiments and results.

\subsection{Methodology}
\label{sec:methodology}
We evaluate active learning separately in the following three frameworks.  \begin{inparaenum}[(i)] 
\item \textbf{Fully supervised}: training a deep network solely on the labeled set, obtained by active queries.
\item \textbf{Fully supervised with self-supervised embeddings}: training a linear classifier on the embedding obtained from a pre-trained self-supervised model. 
\item \textbf{Semi-supervised}: training a deep network on the labeled and unlabeled sets, using the competitive method FlexMatch \cite{DBLP:journals/corr/abs-2110-08263}.
\end{inparaenum} 

In (i) and (ii), we adopt the AL evaluation framework created by \citet{Munjal2020TowardsRA}, which implements several AL methods including all baselines used here except \emph{BADGE}. In (iii) we adopt the code and hyper-parameters provided by FlexMatch. As FlexMatch is computationally intensive, it was not evaluated  on ImageNet, confining the study to datasets it was reported to handle. In all evaluated cases, \emph{TypiClust} achieves large improvements
 (see \app~\ref{app:eval_impl_details} for implementation details). 


We compare \emph{TypiClust} to the following baseline strategies for the selection of $B$ points from $U$: \begin{inparaenum}[(1)] \item \emph{Random} -- uniformly. \item \emph{Uncertainty} -- lowest max softmax output. \item \emph{Margin} -- lowest margin between the two highest softmax outputs. \item \emph{Entropy} --  highest entropy of softmax outputs. \item \emph{DBAL} \citep{gal2017deep}.
\item \emph{CoreSet} \citep{sener2018active}. \item \emph{BALD} \citep{kirsch2019batchbald}. \item \emph{BADGE} \citep{DBLP:conf/iclr/AshZK0A20}. \end{inparaenum} All strategies are evaluated on the following image classification tasks: CIFAR-10/100 \citep{krizhevsky2009learning}, TinyImageNet \citep{le2015tiny} and ImageNet-50/100/200. The latter group includes subsets of ImageNet \citep{deng2009imagenet} containing 50/100/200 classes respectively, following \citet{van2020scan}.


\subsection{Results: Low Budget Regime}
\label{sec:emp_results}
The amount of labeled data that makes a budget ``low'' will vary between tasks. In the following experiments, unlike most earlier work, we focus on scenarios where on average, only 1-10 examples per class are labeled each round.

\subsubsection{Fully Supervised Framework} 
\label{sec:fwork:(i)}

Fig.~\ref{fig:main_al_graph} shows accuracy results for CIFAR-10/100 and ImageNet-100, using the labeled examples queried by different AL strategies. Denoting the number of classes by $M$, we show results with Budget $B=M$ or $B=5M$ labeled examples and $L_0=\emptyset$ (see \app~\ref{app:more_sup_empirical} for additional budgets).

\begin{figure}[hbt]
\begin{center}
\vspace{-.05cm}
\begin{subfigure}{.45\textwidth}
  \centering
 \includegraphics[width=\linewidth]{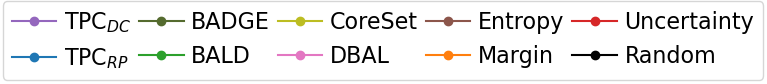}
\end{subfigure}
\\
    \begin{subfigure}{.157\textwidth}
      \centering
      \includegraphics[width=\linewidth]{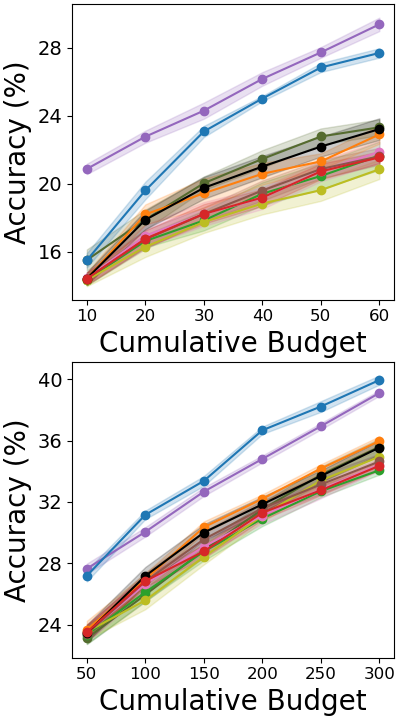}
    \caption{CIFAR-10}
    \label{fig:main_al_graph_cifar10}
    \end{subfigure}
    \begin{subfigure}{.157\textwidth}
      \centering
      \includegraphics[width=\linewidth]{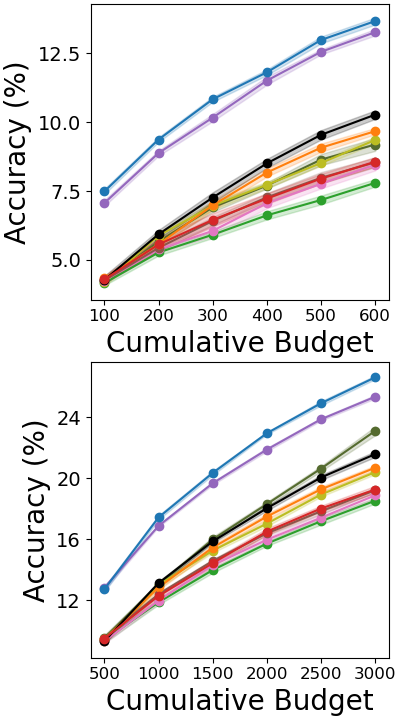}
    \caption{CIFAR-100}
    \label{fig:main_al_graph_cifar100}
    \end{subfigure}
    \begin{subfigure}{.157\textwidth}
      \centering
      \includegraphics[width=\linewidth]{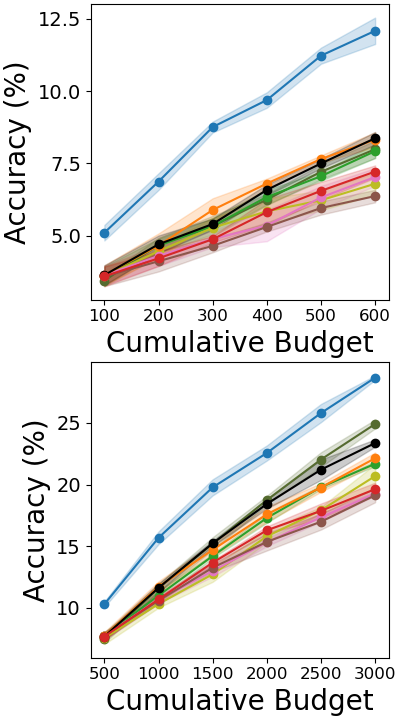}
    \caption{ImageNet-100}
    \label{fig:main_al_graph_tinyimagenet}
    \end{subfigure}
\vspace{-0.35cm}
\caption{``Fully supervised'' framework: comparing \emph{TypiClust} with baseline AL strategies on CIFAR10, CIFAR100, and ImageNet-100 for 5 active learning iterations in the low budget regime. The budget $B$ is equal to (top) the number of classes, or (bottom) $5$ times the number of classes. The final average test accuracy in each iteration is reported, using 10 (CIFAR) and 3 (ImageNet) repetitions. The shaded area reflects standard error.}
\label{fig:main_al_graph}
\end{center}
\vspace{-.25cm}
\end{figure}

We see that in the low budget regime, both \emph{TypiClust} variants outperform the baselines by a large margin. Specifically, all other baseline AL methods perform on par with random selection or worse, in accordance with \citet{pourahmadi2021simple}. In contrast, the typicality-based strategy achieves large accuracy gains.
Noting that most of the baselines are possibly hampered by their use of random initial pool selection when $L_0=\emptyset$, our ablation study in Section~\ref{sec:starting_from_random} demonstrates that this is not a decisive factor.

\subsubsection{Fully Supervised with Self-Supervised Embedding} 
\label{sec:fwork:(ii)}
As self-supervised embeddings can be semantically meaningful, they are often used as features for a linear classifier. Accordingly, in this framework, we use the extracted features from the representation learning step and train a linear classifier on the queried labeled set $L_i$. Unlike the fully supervised framework, here we use the unlabeled data while training the classifier, albeit in a basic manner. This framework outperforms the fully supervised framework, but still lags behind the semi-supervised framework. Once again, \textit{TypiClust} outperforms all baselines by a large margin, as shown in Fig.~\ref{fig:weak_semi} (see \app~\ref{app:more_lin_empirical} for additional datasets). 

\begin{figure}[htb!]
\begin{center}
\begin{subfigure}{.45\textwidth}
  \centering
 \includegraphics[width=\linewidth]{avihu_graphs/maybe_camera_ready/badge_legend.png}
\end{subfigure}
\\
\begin{subfigure}{.157\textwidth}
      \centering
      \includegraphics[width=\linewidth]{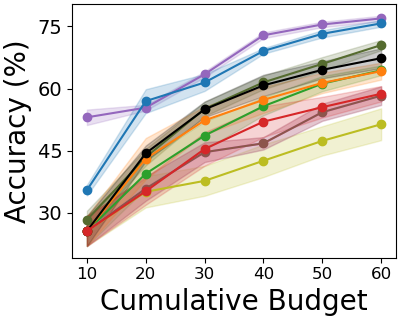}
    \vspace{-0.5cm}
    \caption{CIFAR-10}
    \label{fig:weak_semi_cifar10}
    \end{subfigure}
    \begin{subfigure}{.157\textwidth}
      \centering
      \includegraphics[width=\linewidth]{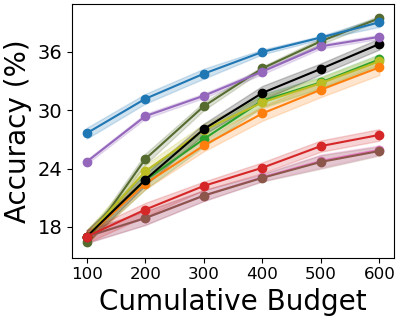}
    \vspace{-0.5cm}
    \caption{CIFAR-100}
    \label{fig:weak_semi_cifar100}
    \end{subfigure}
    \begin{subfigure}{.157\textwidth}
      \centering
      \includegraphics[width=\linewidth]{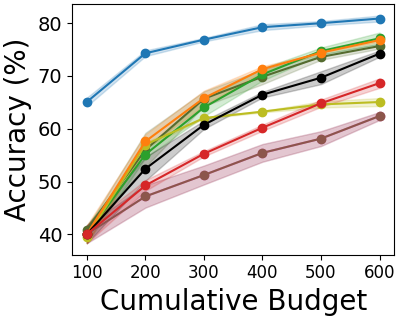}
    \vspace{-0.5cm}
    \caption{ImageNet-100}
    \label{fig:weak_semi_imagenet200}
    \end{subfigure}
    \vspace{-0.35cm}
\caption{Similar to Fig.~\ref{fig:main_al_graph} in the ``fully supervised with self-supervised embedding'' framework
.}
\label{fig:weak_semi}
\end{center}
\vspace{-.3cm}
\end{figure}

\subsubsection{Semi-Supervised Framework} 
\label{sec:fwork:(iii)}

In this framework, we evaluate \emph{TypiClust} and different AL strategies by examining the performance of FlexMatch when trained on their respective queried examples.
As semi-supervised methods often achieve competitive performance with only a few labeled examples, we focus on the extreme low-budget regime, where only $0.02\%\sim1\%$ of the data is labeled. Note that semi-supervised algorithms typically assume a class-balanced labeled set, which is not feasible in active learning. To compare with this scenario which dominates the literature, we add a class-balanced random baseline for reference.

In Fig.~\ref{fig:semi_supervised}, we compare the final performance of FlexMatch using the labeled sets provided by different AL strategies. We show results for a budget of $10$ examples in CIFAR-10 (Fig.~\ref{fig:ssl_cifar_10_with_10_examples}), $300$ examples in CIFAR-100 (Fig.~\ref{fig:ssl_cifar_100_with_300_examples}), and $1000$ examples in TinyImageNet (Fig.~\ref{fig:ssl_tiny_imagenet_with_1000_examples}). We see that both \emph{TypiClust} variants outperform random sampling, whether balanced or not, by a large margin. 
In contrast, other AL baselines do not improve the results of random sampling. Similar results using additional budgets, baselines, datasets, and semi-supervised algorithms, can be found in \app~\ref{app:more_ssl_empirical}.

\begin{figure}[htb!]
\begin{center}
    \begin{subfigure}{.157\textwidth}
      \centering
      \includegraphics[trim={0.7cm 0 0.6cm 0},width=\linewidth]{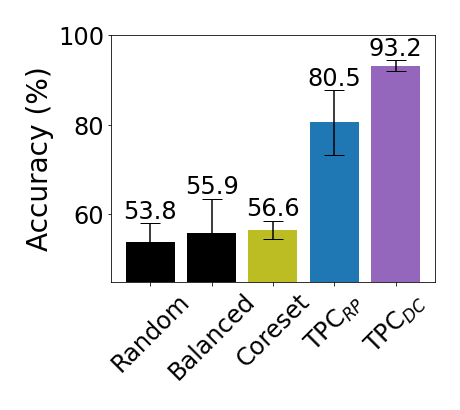}
    \vspace{-0.65cm}
    \caption{CIFAR-10}
    \label{fig:ssl_cifar_10_with_10_examples}
    \end{subfigure}
    \begin{subfigure}{.157\textwidth}
      \centering
      \includegraphics[trim={0.7cm 0 0.6cm 0},width=\linewidth]{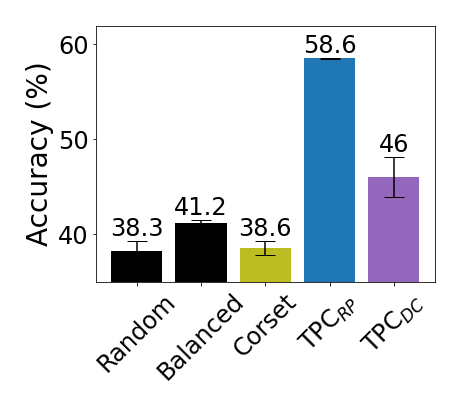}
    \vspace{-0.75cm}
    \caption{CIFAR-100}
    \label{fig:ssl_cifar_100_with_300_examples}
    \end{subfigure}
    \begin{subfigure}{.157\textwidth}
      \centering
      \includegraphics[trim={0.7cm 0 0.6cm 0},width=\linewidth]{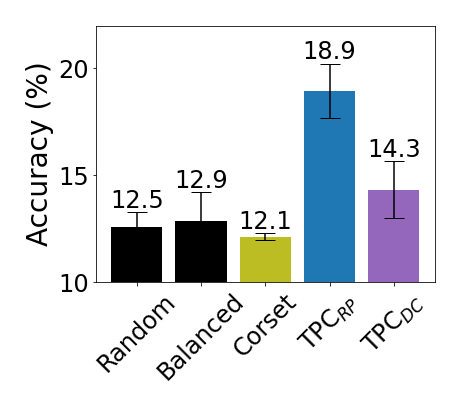}
    \vspace{-0.75cm}
    \caption{TinyImageNet}
    \label{fig:ssl_tiny_imagenet_with_1000_examples}
    \end{subfigure}
\vspace{-0.35cm}
\caption{Comparison of AL strategies in a semi-supervised task. Each bar shows the mean test accuracy after $3$ repetitions of FlexMatch trained on: (a) $10$ examples from CIFAR-10, (b) $300$ examples from CIFAR-100, (c) $1000$ examples from TinyImageNet. Error bars show the standard error. }
\label{fig:semi_supervised}
\vspace{-0.15cm}
\end{center}
\end{figure}

\subsection{Ablation Study}
\label{sec:ablation_study}
We now report the results of a set of ablation studies, checking the added value of each step in our suggested strategy.

\begin{figure}[htb]

\begin{subfigure}{.48\textwidth}
  \centering
 \includegraphics[width=\linewidth]{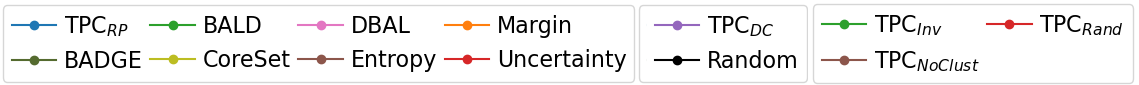}
\end{subfigure}
\\
\begin{subfigure}{.157\textwidth}
  \centering
 \includegraphics[width=\linewidth]{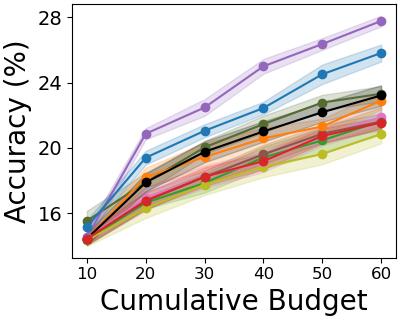}
\caption{}

\label{fig:cifar10_random_init}
\end{subfigure}
\begin{subfigure}{.157\textwidth}
  \centering
 \includegraphics[width=\linewidth]{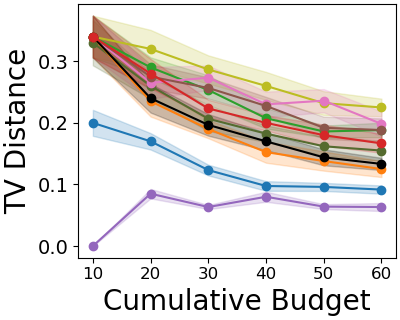}
\caption{}
\label{fig:balancing_metric}
\end{subfigure}
\begin{subfigure}{.157\textwidth}
  \centering
   \includegraphics[width=\linewidth]{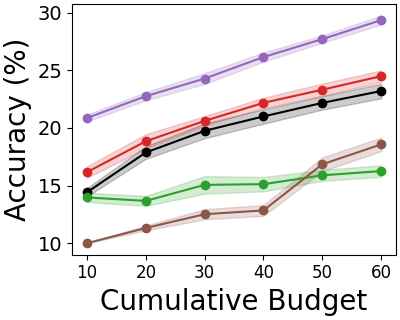}
\caption{}
\label{fig:scores_ablation}
\end{subfigure}

\vspace{-.2cm}
\caption{(a-b) Same experiment as in Fig.~\ref{fig:main_al_graph_cifar10} top, but where: (a) \emph{TypiClust} uses random initial set selection; (b) the Total Variation (TV) distance between the labeled set distribution and the ground truth class distribution is shown. 
(c) To isolate the added value of clustering for diversity and typical sample selection, we evaluate 3 additional selection heuristics on CIFAR-10 (see Section~\ref{sec:abb-denisity}).}
  \label{fig:agreement_testics}
\vspace{-0.25cm}
\end{figure}

\subsubsection{Random Initial Pool Selection}
\label{sec:starting_from_random}

As \textit{TypiClust} is based on self-supervised learning, both its variants are well suited for the case $L_0=\emptyset$, and can actively query the initial selection of labeled examples. By contrast, the other AL baselines use random initial pool selection when $L_0=\emptyset$. To isolate the effect of this difference, we conducted  the same experiment as reported in Fig.~\ref{fig:main_al_graph_cifar10}, giving \emph{TypiClust} a random initial pool selection just like the other baselines. Results are reported in  Fig.~\ref{fig:cifar10_random_init}, showing that \emph{TypiClust} still outperforms all baselines. Importantly, this comparison reveals that non-random initial pool selection yields further generalization gains when combined with active learning. Additional results can be seen in Fig.~\ref{fig:app:random_init_pool}.

\subsubsection{Comparing Class Distribution}
\label{sec:class_balancing}
With an extremely low budget, covering the support of the distribution comprehensively is challenging. To compare the success of the different AL strategies in this task, we measure the Total Variation (TV) distance between the labeled set class distribution and the ground truth class distribution for each strategy. Fig.~\ref{fig:balancing_metric} shows that the \emph{TypiClust} variants achieve a significantly better (lower) score than the alternatives, resulting in queries with better class balance.

\subsubsection{The Importance of Density and Diversity}
\label{sec:abb-denisity}
\emph{TypiClust} clusters the dataset and selects the most typical examples from every cluster. To assess the added value of clustering and typicality selection, we consider the following alternative selection criteria:
\begin{inparaenum}[(a)]
    \item Select a random example from each cluster (TPC$_{Rand}$).
    \item Select the most atypical example in every cluster (TPC$_{Inv}$).
    \item Select typical samples greedily, without clustering (TPC$_{NoClust}$).
\end{inparaenum}

The results in Fig.~\ref{fig:scores_ablation} show that both clustering and high-density sampling are crucial for the success of \emph{TypiClust}. The low performance of TPC$_{Rand}$ shows that representation learning and clustering alone cannot account for all the performance gain, while the low performance of TPC$_{NoClust}$ shows that typicality without diversity is not sufficient (see a  visualization of these variants in Fig.~\ref{fig:clustering_matters}).

\subsubsection{Uncertainty Delivered by an Oracle}
\label{sec:ablation_oracle_uncertainty}

When trained on only a few labeled examples, neural networks tend to overfit, which may result in the unreliable estimation of uncertainty. This offers an alternative explanation to our results -- uncertain examples may be a good choice in the low-budget regime as well, if only we could compute uncertainty accurately.

To test this hypothesis we first train an ``oracle'' network (see \citet{lowell2018practical} and \app~\ref{app:sec:ablation_details}) on the entire CIFAR-10 dataset and use its softmax margin to estimate uncertainty. This ``oracle margin'' is 
then used to choose the query examples. Subsequently, another network is trained similarly to the  setup of Fig.~\ref{fig:main_al_graph_cifar10}, adding in each iteration the examples with either the highest or lowest softmax response margin according to the oracle.

\begin{wrapfigure}{R}{0.24\textwidth}
\centering
\vspace{-0.6cm}
\includegraphics[width=0.24\textwidth]{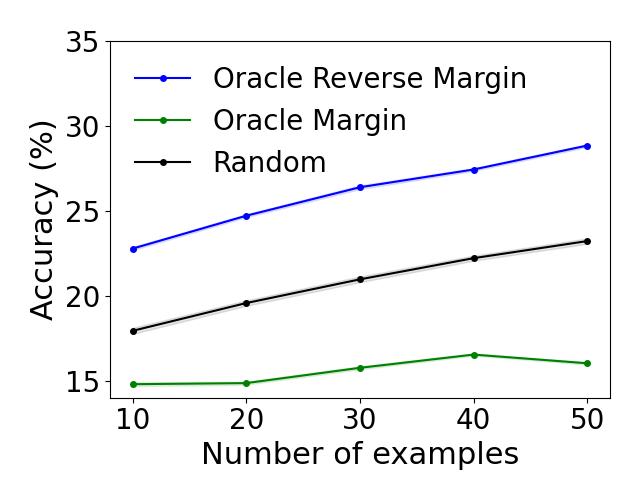}
\vspace{-0.95cm}
\caption{
Certainty, as estimated by the margin of an oracle that knows all the labels, is used for AL. We plot the mean test accuracy of $100$ models trained on CIFAR-10, 
 $|L_0|=10$, $B=10$. STE is very small, as shown.}
\vspace{-0.5cm}
\label{fig:oracle}
\end{wrapfigure}

The results are shown in Fig.~\ref{fig:oracle}
. We see that even a reliable measure of uncertainty leads to poor performance in the low-budget regime, even worse than the baseline uncertainty-based methods. This may be because these methods compute the uncertainty in an unreliable way, and thus behave more like the random selection strategy.

\subsubsection{Imbalanced Data}
Unsupervised representation learning methods often assume class-balanced datasets. As \emph{TypiClust} is based on representation learning, it could potentially fail in imbalanced settings. We repeated our experiments on the class-imbalanced subset of CIFAR-10 proposed by \citet{Munjal2020TowardsRA}
. As before, we show that \emph{TypiClust} outperforms other methods in the low-budget regime, and under-performs in the high-budget regime
 (see low budget results in Fig.~\ref{fig:app:imbalanced_data} of \app~\ref{app:more_sup_empirical}).

\begin{figure}[htb!]
\begin{center}
    \begin{subfigure}{.46\textwidth}
      \centering
      \includegraphics[width=\linewidth]{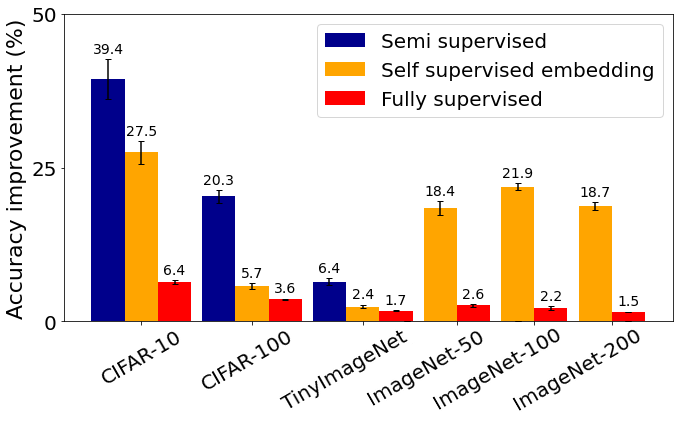}
      \vspace{-0.5cm}
    \end{subfigure}
\vspace{-0.25cm}
\caption{
\textit{TypiClust} achieves major accuracy gains as compared to the random selection baseline in the fully-supervised ($3$ repetitions on ImageNet, $10$ otherwise), semi-supervised ($3$ reps) and self-supervised embedding ($5$ reps) frameworks. We use $10$, $300$, $1000$, $50$, $100$, $200$ examples in CIFAR-10, CIFAR-100, TinyImageNet, ImageNet-50, ImageNet-100 and ImageNet-200 respectively.}
\vspace{-0.7cm}
\label{fig:semi_supervised_vs_fully_acc_improvement}
\end{center}
\end{figure}

\section{Summary and Discussion}
We show, theoretically and empirically, that strategies for active learning in the high and low-budget regimes should be based on opposite principles. Initially, in the low-budget regime, the most typical examples, which the learner can learn most easily, are the most helpful to the learner. When reaching the high-budget regime, the best examples to query are those that the learner finds most confusing. This is the case both in the fully supervised and semi-supervised settings: we show that semi-supervised algorithms get a significant boost from seeing the labels of typical examples. Fig.~\ref{fig:semi_supervised_vs_fully_acc_improvement} summarizes all our empirical results.


Our results are closely related to curriculum learning \citep{bengio2009curriculum,hacohen2019power,weinshall2020theory}, hard data mining, and self-paced learning \citep{kumar2010self}, all of which reflect the added value of typical (``easy'') examples when there is little information about the task, as against atypical (``hard'') examples  which are more beneficial later on. Our results are also closely related to the study of the learning order of neural networks, which also characterize ``easy'' and ``hard'' examples \citet{gissin2019discriminative,hacohen2020let,shah2020pitfalls,hacohen2021principal,choshen2021grammar}.

The point of transition -- what makes the budget ``small'' or ``large'', depends on the task and corresponding data distribution. In complex real-life problems, the low-budget regime may still contain a large number of examples, increasing the practicality of our method. Determining the range of training sizes with ``low budget'' characteristics is a challenging problem, which we leave for future work.

\section*{Acknowledgments}
This work was supported by the Israeli Ministry of Science and Technology, and by the Gatsby Charitable Foundations.


\bibliography{main}
\bibliographystyle{icml2021}


\newpage
.

\newpage
\appendix

\section*{Appendix}

\section{Related Work}
\label{app:related_work}
\subsection{Diversity Sampling}

Deep AL strategies often enforce diversity on the queried batch. The motivation is to avoid redundancy in the annotations, and to represent all parts of the training distribution. \citet{sener2018active} introduced the coreset approach, querying examples that cover the training distribution in a greedy manner. Many other AL algorithms incorporate diversity as part of their sampling strategy, including: \citet{hu2010off,elhamifar2013convex,yang2015multi,wang2016batch,yin2017deep,zhdanov2019diverse,he2019towards,kirsch2019batchbald,DBLP:conf/iclr/AshZK0A20,shui2020deep}. Notably, as deep active learning is practical only in batch settings, the importance of diversity is amplified \citep{geifman2017deep, sener2018active}.

Diversity sampling is orthogonal to uncertainty sampling, and can be added to almost any strategy. As opposed to previous works, our strategy aims to query a diverse set of characteristic examples, while other strategies aim to achieve diverse sets of uncharacteristic examples.

\subsection{AL Strategies for Low Budgets}
Recently, AL in the low-budget regime received increased attention. Strategies designed to address this regime usually employ self-supervised or semi-supervised methods using the unlabeled pool \citep{gao2020consistency,hong2020deep,mahmood2021low,yehuda2022active}. The embedding of such methods is often utilized by methods that estimate uncertainty, as it gives an informative distance measure \citep{zhang2018unreasonable}.

In particular, \citet{DBLP:conf/emnlp/YuanLB20} aims to solve the cold-start problem by using the embedding of a pre-trained model on some unsupervised task. 
Their experiments use pre-trained language models, with a strategy that decreases the dependency on high budgets but is still faithful to uncertainty sampling. 
\citet{mahmood2021low} suggested querying a diverse set of examples with minimal Wasserstein distance from the unlabeled pool. They report a significant performance boost in the low-budget regime. Unlike our work, they conduct experiments only in a fully supervised with self-supervised embedding settings, related to, but somewhat different from, the one described in Section~\ref{sec:fwork:(ii)}.

\section{Mixture Model Lemmas and Proofs}

\subsection{Undulating Error Score: Sufficient Conditions}
\label{app:proof_for_thm_2}

Below we provide the proof for Thm.~\ref{thm:suffiecnt_coniditions_undulating}, which is stated in Section~\ref{sec:error-fs}, and which lists sufficient conditions for error scores to be undulating (see Def.~\ref{def:1}). We start with a few lemmas that will be used in this proof.

\begin{lemma}
\label{lemma:left_equal_1_new}
Let $f:\psn \rightarrow\mathbb{R}$ denote a differentiable function with $f(0)\neq 0$. Then 
\begin{equation*}
\lim_{x\to0^{+}}\frac{f\left(x\right)}{f\left(\ahard x\right)}=1\quad\forall \ahard\in(0,1).
\end{equation*}
\end{lemma}
\begin{proof}
Omitted.
\end{proof}

\begin{lemma}
\label{lemma:right_equal_0_region}
Let $F:\psn \rightarrow \prn$ and $f=F'$ denote a positive differentiable strictly monotonically decreasing function  $(F>0,f<0)$. Assume that $\lim\limits_{x\rightarrow\infty}F(x) =0$, and that the limits $\lim\limits_{x\rightarrow\infty}\frac{F(x)}{F(\ahard x)},\lim\limits_{x\rightarrow\infty}\frac{f(x)}{f(\ahard x)}$ exist $\forall \ahard\in(0,1)$. Denote $g(x)=-\ln\left(F(x)\right)$. If 
\begin{equation*}
g'(x)\in\omega\left(\frac{1}{x}\right),
\end{equation*}
then:
\begin{equation*}
\lim_{x\rightarrow\infty}\frac{f(x)}{f\left(\ahard x\right)}=0 \quad\forall \ahard\in(0,1).
\end{equation*}
\end{lemma}

\begin{proof}
We can write $F(x)=e^{-g(x)}$.
It follows from the mean value theorem that $\exists t~\ahard x < t < x$ such that
\begin{align*}
\frac{F(x)}{F(\ahard x)} &= e^{-(g(x)-g(\ahard x))} \\
&= e^{-(g'(t)x(1-\ahard))}\\
&= e^{-(g'(t)t\cdot\frac{x}{t}\cdot(1-\ahard))}.
\end{align*}
Since $g'(x)\in\omega\left(\frac{1}{x}\right)$ we get
\begin{align*}
\lim_{t\rightarrow\infty}t\cdot g'(t)=\infty.
\end{align*}
As $(1-a)<\frac{x}{t}(1-a)<\frac{(1-a)}{a}$, it follows that
\begin{equation*}
\lim_{x\rightarrow\infty}\frac{F(x)}{F(\ahard x)}=0.
\end{equation*}
From the assumption that the limits exist, and since $\lim\limits_{x\rightarrow\infty}F(x) =\lim\limits_{x\rightarrow\infty}F(\ahard x) =0$, we can use L'Hôpital's rule and get
\begin{equation*}
\lim_{x\rightarrow\infty}\frac{F(x)}{F(\ahard x)}=\lim_{x\rightarrow\infty}\frac{f(x)}{\ahard f(\ahard x)}=\frac{1}{a}\lim_{x\rightarrow\infty}\frac{f(x)}{f(\ahard x)}=0.
\end{equation*}
\end{proof}

\begin{lemma}
\label{lemma:bound_result_in_omega}
Let $F:\psn \rightarrow \prn$ denote a positive differentiable function $(F>0)$. Denote $g(x)=-\ln(F(x))$. Assume that
$\lim\limits_{x\to\infty}g'(x)$ exists, and
\begin{equation*}
g(x)\in\omega\left(\log(x)\right),
\end{equation*}
then
\begin{equation*}
g'(x)\in\omega\left(\frac{1}{x}\right).
\end{equation*}
\end{lemma}

\begin{proof}
$g(x)\in\omega(\log(x))$ implies that\;
\begin{equation*}
\lim_{x\to\infty}\frac{g(x)}{\ln(x)}=\infty,~ \lim_{x\to\infty}g(x)=\infty.
\end{equation*}
We can now use L'Hôpital's rule and get
\begin{equation*}
\infty=\lim_{x\to\infty}\frac{g(x)}{\ln(x)}=\lim_{x\to\infty}\frac{g'(x)}{\frac{1}{x}}=\lim_{x\to\infty}xg'(x).
\end{equation*}
\end{proof}

\noindent
\textbf{Theorem~\ref{thm:suffiecnt_coniditions_undulating}.} \textit{
Given partition $R_\pR\prec R_\rR$ and Error score $E(x)$, let $p=Prob(R_\pR)$ and $0<\qhard<\frac{p}{1-p}$. $E(x)$ is undulating if the following assumptions hold:
\begin{enumerate}[label=(\roman*)]
\item $E(x)$ is a proper \emph{error score} (see Def.~\ref{def:proper}).
\item $\lim\limits_{x\to\infty}\frac{E'(x)}{E(x)},\lim\limits_{x\rightarrow\infty}\frac{E(x)}{E(\ahard x)},\lim\limits_{x\rightarrow\infty}\frac{E'(x)}{E'(\ahard x)}$ exist $\forall\ahard\!\in\! (0,1)$.
\item 
$-\log(E(x))\in\omega\left(\log(x)\right)$.
\end{enumerate}
}

\begin{proof}
We define $f(x)=-E'(px), \ahard=\frac{\qhard(1-p)}{p}<1$. From assumption (i) and using Lemma~\ref{lemma:left_equal_1_new}, we get
\begin{equation*}
\lim_{x\to 0^+}\frac{E'\left(px\right)}{E'\left(\qhard(1-p)x\right)}=1.
\end{equation*}
Therefore there exists some $z_1\in\psn$ such that $\forall x<z_1$
\begin{equation*}
\frac{E'\left(px\right)}{E'\left(\qhard(1-p)x\right)}>\frac{\qhard(1-p)}{p}.
\end{equation*}
From assumptions (i)-(iii) and using Lemmas~\ref{lemma:right_equal_0_region}-\ref{lemma:bound_result_in_omega}, we get
\begin{equation*}
\lim_{x\to\infty}\frac{E'\left(x\right)}{E'\left(ax\right)}=\lim_{x\to\infty}\frac{E'\left(px\right)}{E'\left(\qhard(1-p)x\right)}=0,
\end{equation*}
and therefore there is some $z_2\in\psn$ such that $\forall x>z_2$
\begin{equation*}
\frac{E'\left(px\right)}{E'\left(\qhard(1-p)x\right)}<\frac{\qhard(1-p)}{p}.
\end{equation*}
From Def.~\ref{def:1} we get that $E(x)$ is undulating.
\end{proof}

\subsection{SP-undulating Error Score: Sufficient Conditions}
\label{app:proof_thm_3}

We provide next the proof for Thm.~\ref{thm:suffiecnt_coniditions_sp_undulating}, which is stated in Section~\ref{sec:SP_undulating}, and which lists sufficient conditions for error scores to be SP-undulating (see Def.~\ref{def:1}), extending Thm.~\ref{thm:suffiecnt_coniditions_undulating}. Once again, we start with a few lemmas.

\begin{lemma}
\label{thm:1_crossing_point}
Let $f:\psn\rightarrow\prn$ denote a positive differentiable function $(f>0)$. Let $0 < \ahard < 1$ denote some constant. If
\begin{equation*}
h\left(x\right)=\frac{-f'\left(x\right)x}{f\left(x\right)}
\end{equation*}
is strictly monotonically increasing, then
\begin{equation*}
g\left(x\right)=\frac{f\left(x\right)}{f\left(\ahard x\right)}
\end{equation*}
is strictly monotonically decreasing.
\end{lemma}

\begin{proof}

$g\left(x\right)$ is monotonically decreasing iff $g'\left(x\right)<0$, where
\begin{align*}
g'\left(x\right) & =\frac{f'\left(x\right)f\left(\ahard x\right)-\ahard f\left(x\right)f'\left(\ahard x\right)}{f\left(\ahard x\right)^{2}}.
\end{align*}
This condition translates to
\begin{equation*}
f'\left(x\right)f\left(\ahard x\right)-\ahard f\left(x\right)f'\left(\ahard x\right)<0.
\end{equation*}
As by assumption $h\left(x\right)$ is monotonically increasing and $h\left(\ahard x\right)  < h\left(x\right)$, we get that $\forall x>0$
\begin{align*}
\quad &\frac{-f'\left(\ahard x\right)\ahard x}{f\left(\ahard x\right)}  <\frac{-f'\left(x\right)x}{f\left(x\right)} \\
\implies \quad &\frac{-f'\left(\ahard x\right)\ahard}{f\left(\ahard x\right)}  <\frac{-f'\left(x\right)}{f\left(x\right)} \\
\implies \quad &-f'\left(\ahard x\right)f\left(x\right)\ahard  <-f'\left(x\right)f\left(\ahard x\right).
\end{align*}
\end{proof}

\begin{lemma}
\label{thm:log-concave_is_monotinic}
Let $f:\psn\rightarrow\prn$ denote a positive differentiable log-concave function which is strictly monotonically decreasing $(f>0,f'<0,(\log(f))''\leq0)$. Then the following function is strictly monotonically increasing
\begin{equation*}
h(x)=\frac{-f'\left(x\right)x}{f\left(x\right)}.
\end{equation*}
\end{lemma}
\begin{proof}
$h(x)$ is strictly monotonically increasing iff $h'(x)>0$, which holds iff
\begin{align*}
h'(x)=\frac{xf'(x)^2-xf(x)f''(x)-f(x)f'(x)}{f(x)^2}&>0\\
\implies \quad \hspace{0.18cm} x\left[f'(x)^2-f(x)f''(x)\right]-f(x)f'(x)&>0.\\
\end{align*}
Recall that $x>0$ and $-f(x)f'(x)>0~\forall x\in \psn$. Since $f$ is log-concave, we also have that $f'(x)^2-f(x)f''(x)\geq0$, which concludes the proof.
\end{proof}

\noindent
{\textbf{Theorem~\ref{thm:suffiecnt_coniditions_sp_undulating}.} \textit{Given partition $R_\pR\prec R_\rR$ and Error score $E(x)$, let $p=Prob(R_\pR)$ and $0<\qhard<\frac{p}{1-p}$.  $E(x)$ is SP-undulating if the following assumptions hold:}
\begin{enumerate}
\item $E(x)$ is an undulating proper \emph{error score}.
\label{ass:1}
\item At least one of the following conditions holds:
\label{ass:2}
\begin{enumerate}
\item \label{ass:2-1}
$\frac{-E''\left(x\right)\cdot x}{E'\left(x\right)}$ is
monotonically increasing with $x$.
\item \label{ass:2-2}
$-E'(x)$ is strictly monotonic decreasing and log-concave. 
\end{enumerate}

\end{enumerate}
}

\begin{proof}
Define the following positive continuous function
\begin{equation*}
H\left(x\right)=\frac{E'\left(px\right)}{E'\left(\qhard\left(1-p\right)x\right)}.
\end{equation*}
Let $f(x)=-E'(x), ~a=\frac{\ahard(1-p)}{p}<1$. Note that assumption \ref{ass:2-1} follows from assumption \ref{ass:2-2} and Lemma~\ref{thm:log-concave_is_monotinic}.
Assumption \ref{ass:2} therefore implies that $\frac{-f'(x)x}{f(x)}$ is strictly monotonically increasing, and by Lemma~\ref{thm:1_crossing_point} we can conclude that $H(x)$ is monotonically decreasing. Together with assumption \ref{ass:1}, $H\left(x\right)=\frac{\qhard\left(1-p\right)}{p}$ at a single point, and we may therefore conclude that $E(x)$ is SP-undulating.
\end{proof}

\begin{corollary}
If $p$ -- the probability of region $R_\pR$ -- is sufficiently small so that $p<\frac{\qhard}{1+\qhard}$, then the conclusions are reversed: it is beneficial to initially over-sample $R_\rR$, and vice versa. 
\end{corollary}

\section{Error Function of Simple Mixture Models}
\label{app:simple_mixture_models_error_func}

\subsection{Mixture of Two Linear Classifiers}
\label{app:linear-model}

We next prove Thm.~\ref{thm:error_bound_pw_linear} and Thm.~\ref{thm:undulating_pw_linear}, which are stated in Section~\ref{sec:piecewise_linear_separator}. Thm.~\ref{thm:error_bound_pw_linear} provides a bound on the error score of a single linear classifier, showing that under mild conditions, this score is bounded by an exponentially decreasing function in the number of training examples $m$. Thm.~\ref{thm:undulating_pw_linear} states  conditions under which the error score of a mixture of two linear models $E(m)$ is undulating, and presents the phase-transition behavior. 



\subsubsection{Bounding the Error of Each Mixture Component}
Henceforth we use the notations of Section~\ref{sec:piecewise_linear_separator}, where for clarity, $m_j$ is replaced by $m$ while we are discussing the bound on a single component $j\in[2]$. Let $\bx_i$ denote the $i$-th data point and $i$-th column of $X$. Let $\mu_1$ and $\mu_2$ denote the respective means of the two classes, and $\mu=\mu_1-\mu_2$ denote the vector difference between the means. 

Assuming that the data is normalized to $0$ mean, the maximum likelihood estimators for the covariance matrix of the distribution $\Sigma$ and class means, denoted $\hat \Sigma$ and $\hat \mu_1,\hat \mu_2$ respectively, are the following
\begin{align}
\hat \Sigma &= \frac{1}{m}\sum_{i=1}^m \bx_i \bx_i^\top = \frac{1}{m} XX^\top \label{eq:sigma-hat}\\
\hat \mu_j &= \frac{1}{m_j}\sum_{\bx_i\in C_j} \bx_i \implies {\by}X^\top=m_1 \hat\mu_1-m_2\hat\mu_2 \label{eq:mu-hat},
\end{align}
where $C_j$ denotes the set of points in class $j\in [2]$. Thus, the ML linear separator can be written as
\begin{equation*}
\begin{split}
\hat {\bw} &= {\by}X^\top (XX^\top)^{-1} \\
&= [m_1 \hat\mu_1-m_2\hat\mu_2]~(m \hat \Sigma)^{-1} = \hat\mu \hat\Sigma^{-1},
\end{split}
\end{equation*}
where $\hat \mu =  \frac{1}{m} \by X^\top$. Note that $\hat\mu$ is the sample mean of vectors $\{y_i\bx_i\}_{i=1}^m$, and $\hat\Sigma$ is the sample covariance of $\{\bx_i\}_{i=1}^m$.

When $d>m$ (fewer training points than the input space dimension), $\hat\Sigma$ is rank deficient and therefore $\hat\Sigma^{-1}$ is not defined. Moreover, the solution is not unique. Nevertheless, it can be shown that the minimal norm solution is the Penrose pseudo-inverse $\hat\Sigma^{+}$, where
\begin{equation*}
\hat\Sigma = U D U^\top \implies \hat\Sigma^{+} = U D^+ U^\top,
\end{equation*}
and using the notations
\begin{align*}
D &= \mathrm{diag}(d_1,\ldots,d_m, 0,\ldots, 0) \\
D^+ &= \mathrm{diag}(d_1^{-1},\ldots,d_m^{-1}, 0,\ldots, 0).
\end{align*}

Ignoring the question of uniqueness, estimating ${\bw}$ is therefore reduced to evaluating the estimators in (\ref{eq:sigma-hat}) and (\ref{eq:mu-hat}). These ML estimators have the following known upper bounds on their error:
\begin{enumerate}
\item Bounding $\hat\Sigma$: from known results on covariance estimation \citep{DBLP:journals/ftml/Tropp15}, 
using Bernstein matrix inequality 
\begin{equation}
\label{eq:cov-bound}
    P (\|\hat\Sigma-\Sigma\|_{op} \geq t) \leq 2d e^{-\gamma m t^2}.
\end{equation}
Constant $\gamma$ does not depend on $m$; it is determined by the assumed bound on the $L_2$ norm of vectors $\bx_i$, and the norm of the true covariance matrix $\Sigma$.
\item Bounding $\hat\mu$: starting from Hoeffding's inequality in one dimension, we have that $P (\left |\hat\mu^{k}-\mu^{k}\right | \geq t) \leq 2e^\frac{-2 m t^2}{4\capB^2}$ $\forall k\in [d]$, where we assume a bounded distribution $\|\bx\|\le \capB$. Thus 
\begin{equation}
\label{eq:mean-bound}
\begin{split}
P (\|\hat\mu-\mu\| \geq t) &= P (\|\hat\mu-\mu\|^2 \geq t^2) \\
&= P (\sum_{k=1}^d\left |\hat\mu^{k}-\mu^{k}\right |^2 \geq t^2) \\
&\leq \sum_{k=1}^d P (\left |\hat\mu^{k}-\mu^{k}\right |^2 \geq \frac{t^2}{d})\\
&= \sum_{k=1}^d P (\left |\hat\mu^{k}-\mu^{k}\right | \geq \frac{t}{\sqrt{d}}) \\
&\leq 2 d e^{-2 m \frac{t^2}{4\capB^2 d}}
\end{split}
\end{equation}
\end{enumerate}
The first inequality follows from the union-bound inequality.



\begin{lemma}
\label{thm:new_lemma}
\begin{equation}
\label{eq:simps}
\hat\mu^\top [\Sigma^{-1}-\hat\Sigma^{+}]\bx =\hat\mu^\top  [\hat\Sigma^{+}(\hat\Sigma-\Sigma)\Sigma^{-1}]\bx.
\end{equation}
\end{lemma}

\begin{proof}
Because $\hat\Sigma\in\R^{d\times d}$ is of rank $m$, $Q= \hat\Sigma^{+}\hat\Sigma=U \mathrm{diag}(1,\ldots,1, 0,\ldots, 0)U^\top$ is a projection matrix of rank $m$, projecting vectors to the subspace spanned by the training set $\{\bx_i\}_{i=1}^m$. Thus $Q\hat\mu=\hat\mu$. Additionally, by definition, $\hat\Sigma^{+}\hat\Sigma\hat\Sigma^{+} = \hat\Sigma^{+}$ and $Q$ is symmetric. It follows that
\begin{equation*}
\begin{split}
\hat\mu^\top (\Sigma^{-1}-\hat\Sigma^{+})\bx
&= \hat\mu^\top  (Q\Sigma^{-1}-\hat\Sigma^{+})\bx,
\end{split}
\end{equation*}
while
\begin{equation*}
\hat\Sigma^{+}(\hat\Sigma-\Sigma)\Sigma^{-1} = Q\Sigma^{-1}-\hat\Sigma^{+}.
\end{equation*}
Together, we get (\ref{eq:simps}).
\end{proof}

\noindent
\textbf{Theorem~\ref{thm:error_bound_pw_linear}.} \textit{
Assume: \begin{inparaenum}[(i)] \item a bounded sample $\|\bx_i\|\le \capB$, where $X X^\top$ is sufficiently far from singular so that its smallest eigenvalue is bounded from below by $\frac{1}{\Lambda}$; \item a realizable binary problem where the classes are separable by margin $\delta$; \item full rank data covariance, where $\frac{1}{\lambda}$ denotes its smallest singular value. \end{inparaenum} Then there exist some positive constants $k, \nu > 0$, such that $\forall m_j\in\mathbb{N}$ and every sample $\iX^{m_j}$, the expected error of $\hat {\bw}$ obtained using $\iX^{m_j}$ is bounded by:
\begin{equation*}
F({m_j}) = \E_{\bX\sim\Dd_j}[\mathrm{0-1~loss~of~}\hat {\bw}] \leq ke^{-\nu {m_j}}
\end{equation*}
}

\label{app:proof_thm_4}
\begin{proof}
It follows from our assumptions that $\|\hat\Sigma^{+}\|_{op}\leq\Lambda$. An error will occur at $\bx\in C_1$ if $\Wo\bx\geq\delta$ and $\hat {\bw}\bx< 0$, and vice versa for $\bx\in C_2$. In either case, the difference between the predictions of $\Wo$ and $\hat {\bw}$ deviate by more than $\delta$. Thus
\begin{equation*}
\begin{split}
F(m) &= \E_{\bX\sim\Dd}[\mathrm{0-1~loss}] = P(\mathrm{error})\\
&\leq  P \left(\big|\mu^\top\Sigma^{-1}\bx - \hat\mu^\top \hat\Sigma^{+}\bx\big | > \delta\right).
\end{split}
\end{equation*}
Invoking the triangular inequality
\begin{equation}
\label{eq:1}
\begin{split}
\Delta&=\big |\mu^\top\Sigma^{-1}\bx - \hat\mu^\top \hat\Sigma^{+}\bx\big | \\
&=\big | (\mu-\hat\mu)^\top\Sigma^{-1}\bx + \hat\mu^\top (\Sigma^{-1} - \hat\Sigma^{+})\bx \big |  \\
& \leq\big | (\mu-\hat\mu)^\top\Sigma^{-1}\bx \big | ~+~ \big |\hat\mu^\top (\Sigma^{-1} - \hat\Sigma^{+})\bx \big |.
\end{split}
\end{equation}

We now use Lemma~\ref{thm:new_lemma}  in order to shift $(\Sigma^{-1} - \hat\Sigma^{+})$ to $(\Sigma - \hat\Sigma)$. Specifically, we insert (\ref{eq:simps})  into (\ref{eq:1}) to get
\begin{align*}
\Delta
&\leq \left\| \mu-\hat\mu\right\| \big\|\Sigma^{-1}\big\|_{op} \left\|\bx \right\| \\
&\quad + \left\|\hat\mu\right\| \big\|  \hat\Sigma^{+} \big\|_{op} \big\|\hat\Sigma-\Sigma\big\|_{op}  \big\|\Sigma^{-1}\big\|_{op} \left\|\bx\right\|\\
&\leq \left\| \mu-\hat\mu\right\| \lambda \capB ~+~ \capB \Lambda \big\|\hat\Sigma-\Sigma\big\|_{op} \lambda \capB.
\end{align*}
It follows that
\begin{equation*}
\begin{split}
P&\Big (\big |\mu^\top\Sigma^{-1}\bx - \hat\mu^\top \hat\Sigma^{+}\bx\big | > \delta\Big ) \\
&\leq P(\left\| \mu-\hat\mu\right\| \lambda \capB +  \big\|\hat\Sigma-\Sigma\big\|_{op} \lambda\Lambda \capB^2 > \delta) \\ 
&\leq P(\left\| \mu-\hat\mu\right\| \lambda \capB > \frac{\delta}{2}) + P( \big\|\hat\Sigma-\Sigma\big\|_{op} \lambda\Lambda \capB^2 > \frac{\delta}{2}) \\ 
&\leq 2d e^{-2 m \left (\frac{\delta}{2\lambda \capB}\right )^2 \frac{1}{4\capB^2 d}} + 2d e^{-a m \left (\frac{\delta}{2\lambda\Lambda \capB^2}\right )^2} \\
&\leq 4d e^{-k m \delta^2}.
\end{split}
\end{equation*}
The second transition follows from the union-bound inequality, and the third from~(\ref{eq:cov-bound})-(\ref{eq:mean-bound}) (where constant $\gamma$ is defined). For the last transition we define
\begin{equation*}
k=\min\left\{\frac{2}{4\capB^2 d}\left(\frac{1}{2\lambda \capB}\right )^2,\gamma \left(\frac{1}{2\lambda\Lambda \capB^2}\right)^2\right\}.
\end{equation*}
\end{proof}

\subsubsection{A Mixture Classifier}
\label{app:proof_of_thm_5}

Assume a mixture of two linear classifiers, and let $E(m)=p\cdot E_{\Dd_\pR}(m_\pR) + (1-p)\cdot E_{\Dd_\rR}(m_\rR)$ denote its error score. 

\noindent
\textbf{Theorem~\ref{thm:undulating_pw_linear}.} \textit{
Keep the assumptions stated in Thm.~\ref{thm:error_bound_pw_linear}, and assume in addition that $\forall\ahard\!\in\! (0,1)$, $\exists\lim\limits_{m\to\infty}\frac{E'(m)}{E(m)}, \lim\limits_{m\rightarrow\infty}\frac{E(m)}{E(\ahard m)} \lim\limits_{m\rightarrow\infty}\frac{E'(m)}{E'(\ahard m)}$. 
Then the error score of a mixture of two linear classifiers is undulating. 
}
\begin{proof}
In each region $j$ of the mixture, Thm.~\ref{thm:error_bound_pw_linear} implies that its corresponding error score as defined in Def.~\ref{def:expected-error},
\begin{eqnarray*}
E_{\Dd_j}(m_j) = \E_{\iX^{m_j}\sim\Dd_j^{m_j}}\left[F(m_j)\right]
\end{eqnarray*}
is bounded by an exponentially decreasing function of $m_j$. Since $E_{\Dd_j}(m_j)$ measures the expected error over all samples of size $m_j$, it can also be shown that $E_{\Dd_j}(m_j)$ is monotonically decreasing with $m_j$. From the separability assumption, $\lim_{m_j\rightarrow\infty}E_{\Dd_j}(m_j)=0$. Finally, since $E(m)$ is a linear combination of two such functions, it also has these properties. We conclude from Cor.~\ref{cor:exp_undulating} that the error score of a mixture of two linear classifiers is undulating.
\end{proof}

\subsection{1-NN Classifier}
\label{app:sec:1_nn_classifier_full_sec}
If the training sample size $m$ is small, our analysis shows that under certain circumstances, it is beneficial to prefer sampling from a region $R$ where $E_{\Dd_{R}}(m)<E_{\Dd_{\Omega \setminus R}}(m)$. We now show that the set of densest points  has this property.

To this end, we adopt the one-Nearest-Neighbor (1-NN) classification framework. This is a natural framework to address the aforementioned question for two reasons: \begin{inparaenum}[(i)] \item  It involves a general classifier with desirable asymptotic properties.
\item The computation of both class density and 1-NN is governed by local distances in the input feature space.\end{inparaenum} 

To begin with, assume a classification problem with $k$ classes that are separated by at least $\rho$. More specifically, assume that $\forall \bx,\bx'\in\mathbb{R}^{d}$, if $\|\bx'-\bx\|\leq\rho$ then  $y'=y$. Let $B_{\calligra v}(\bx_i,r)$ denote a ball centered at $\bx_i\in\R^d$, with radius smaller than $r$ and volume ${\calligra v}$. For $\bX=(\bx,y)$, let $f_\Dd(\bX)$ denote the density function from which data is drawn when sampling is random (and specifically at test time).

Assume a 1-NN classifier whose training sample is $T=\{\bx_i,y_i\}_{i=1}^m$, and whose prediction at $\bX=(\bx,y)$ is 
\begin{equation*}
y=\begin{cases}
y_\nu,~~ \nu=\mathrm{arg}\min\limits_{i\in [m]}\|\bx-\bx_i\| & \quad   \bx\in B_{\calligra v}(\bx_i,\rho) \\
y\sim U(1,k) & \quad\mathrm{otherwise}\\
\end{cases}
\end{equation*}

The error probability of this classifier at $\bx$ is
\begin{equation*}
P(x)=\begin{cases}
0 & \quad \exists \mathrm{~i~such~that~} \bx\in B_{\calligra v}(\bx_i,\rho)\\
\frac{k-1}{k} & \quad\mathrm{otherwise}
\end{cases}
\end{equation*}
The $0-1$ loss of this classifier is
\begin{equation}
\label{eq:union}
\E_{\bX\sim\Dd}[P(x)]=\frac{k-1}{k} Prob\left [\bx\notin \bigcup_{i=1}^{m} B_{\calligra v}(\bx_i,\rho) \right ],
\end{equation}
where $B_{\calligra v}(\bx_i,r)$ denotes a ball centered at $\bx_i\in\R^d$, with radius smaller than $r$ and volume ${\calligra v}$. 

The next theorem states properties of set $T$ which are beneficial to the minimization of this loss:

\begin{theorem}
\label{thm:expected_error_1nn}
Let $A_i$ denote the event $\left \{\bx\in B_{\calligra v}(\bx_i,r)\right \}$, and assume that these events are independent. Then we have
\begin{equation*}
L(T) = \frac{k-1}{k} \left [1- \sum_{i=1}^m  f_\Dd(\bX_i){\calligra v} + O({\calligra v}^2)\right ].
\end{equation*}
\end{theorem}

\begin{proof}
Using the independence assumption and (\ref{eq:union}), and assuming that  ${\calligra v}$ is sufficiently small
\begin{equation*}
\begin{split}
L(T) &= \frac{k-1}{k} \left [ 1- P\left (\bigcup_{i=1}^{m} A_i\right ) \right ] \\
&=\frac{k-1}{k} \left [ 1- \sum_{i=1}^m P\left (A_i\right ) \right ]\\
&=\frac{k-1}{k} \left [1- \sum_{i=1}^m  \int_\bX \mathbbm{1}_{{\textstyle\mathstrut} {\bx\in B_{\calligra v}(\bx_i,r)}} f_\Dd(\bX)d\bX\right ]\\
&= \frac{k-1}{k} \left [1- \sum_{i=1}^m  f_\Dd(\bX_i){\calligra v} + O({\calligra v}^2)\right ].
\end{split}
\end{equation*}
\end{proof}

In Thm.~\ref{thm:expected_error_1nn}, we show that if ${\calligra v}$ is sufficiently small, the 0-1 loss is minimized when choosing a set of independent points $\{\bX_i\}_{i=1}^m$ that maximizes $\sum_{i=1}^m  f_\Dd(\bX_i)$. This suggests that the selection of an initial pool of size $m$ will benefit from the following heuristic: 
\setlist{nolistsep}
\begin{itemize}[noitemsep]
    \item \textbf{Max density:} when selecting a point $\bX_i$, maximize its density $f_\Dd(\bX_i)$.
    \item \textbf{Diversity:} select varied points, for which the events $\left \{\bx\in B_{\calligra v}(\bx_i,\rho)\right \}$ are approximately independent.
\end{itemize}

\section{Theoretical Analysis: Visualization}

\subsection{Mixture of Two Linear Models}

We now empirically analyze the error score function of the mixture of two linear classifiers, defined in Section~\ref{sec:piecewise_linear_separator}. Each linear classifier is trained on a different area of the support. The data is $100$ dimensional, linearly separable in each region. The margin is used to determine the $\qhard$ of the data. The data is chosen such that $p=0.9,\qhard=0.2$. The results are shown in Fig.~\ref{fig:error_linear_classifier_mixture}.

\begin{figure}[thb!]
    \begin{subfigure}{.23\textwidth}
      \centering
        \includegraphics[width=1\linewidth]{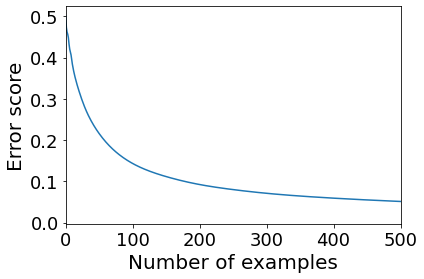}
        \caption{Error score}
    \end{subfigure}
    \begin{subfigure}{.23\textwidth}
      \centering
        \includegraphics[width=1\linewidth]{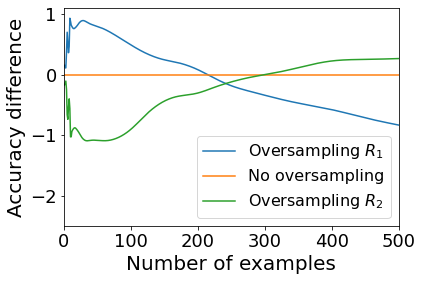}
        \caption{Phase transition}
      \label{subfig:error_linear_classifier_mixture_b}
    \end{subfigure}
\caption{(a) The error score $E(m)$ as a function of the number of examples, averaged over $10k$ repetitions. While the error score is not exponential, it could be upper bounded by an exponential function, as analytically shown in Section~\ref{sec:piecewise_linear_separator}. (b) The differences in accuracy when over-sampling from either $R_\pR$ and $R_\rR$ over a random sampling from the data distribution. Although the error score is only proven to be undulating, we can see that in practice it is also SP-undulating.}
\label{fig:error_linear_classifier_mixture}
\end{figure}

\begin{figure}[thb!]
\includegraphics[width=4cm,height=3.5cm]{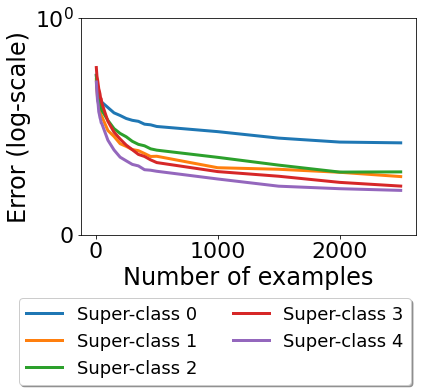}
\includegraphics[width=4cm,height=3.5cm]{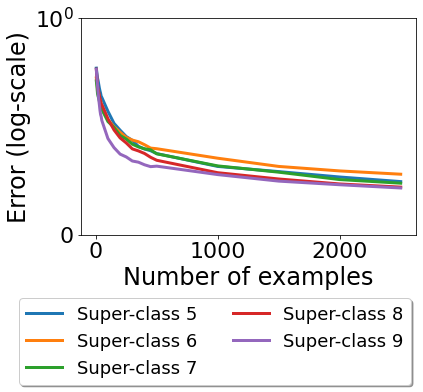}

\includegraphics[width=4cm,height=3.5cm]{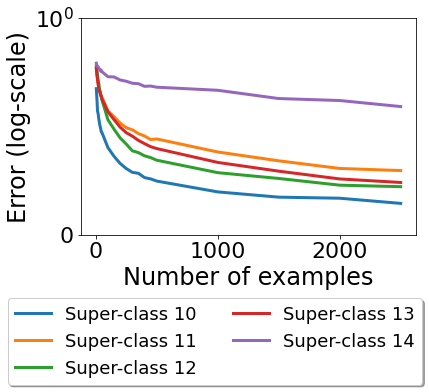}
\includegraphics[width=4cm,height=3.5cm]{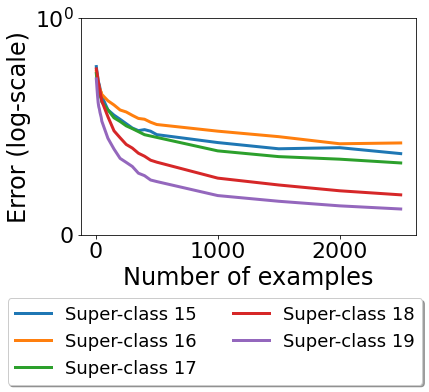}
\caption{Log-error scores of image classification datasets as a function of the number of training examples. Each score is calculated as $1-accuracy$, averaged on $100$ VGG-16 networks trained on super-classes of CIFAR-100. Each line represents the error of a different super-class. As the log of the error is plotted, it can be seen that in all cases the error scores are monotonic decreasing and can be bounded above by some exponential function, suggesting that often the assumptions in Thm.~\ref{thm:suffiecnt_coniditions_sp_undulating}  hold.}
\label{fig:error_scores_neural_networks}
\end{figure}

\begin{figure*}[thb!]
\begin{center}
    \begin{subfigure}{.27\textwidth}
      \centering
      \includegraphics[width=1\linewidth]{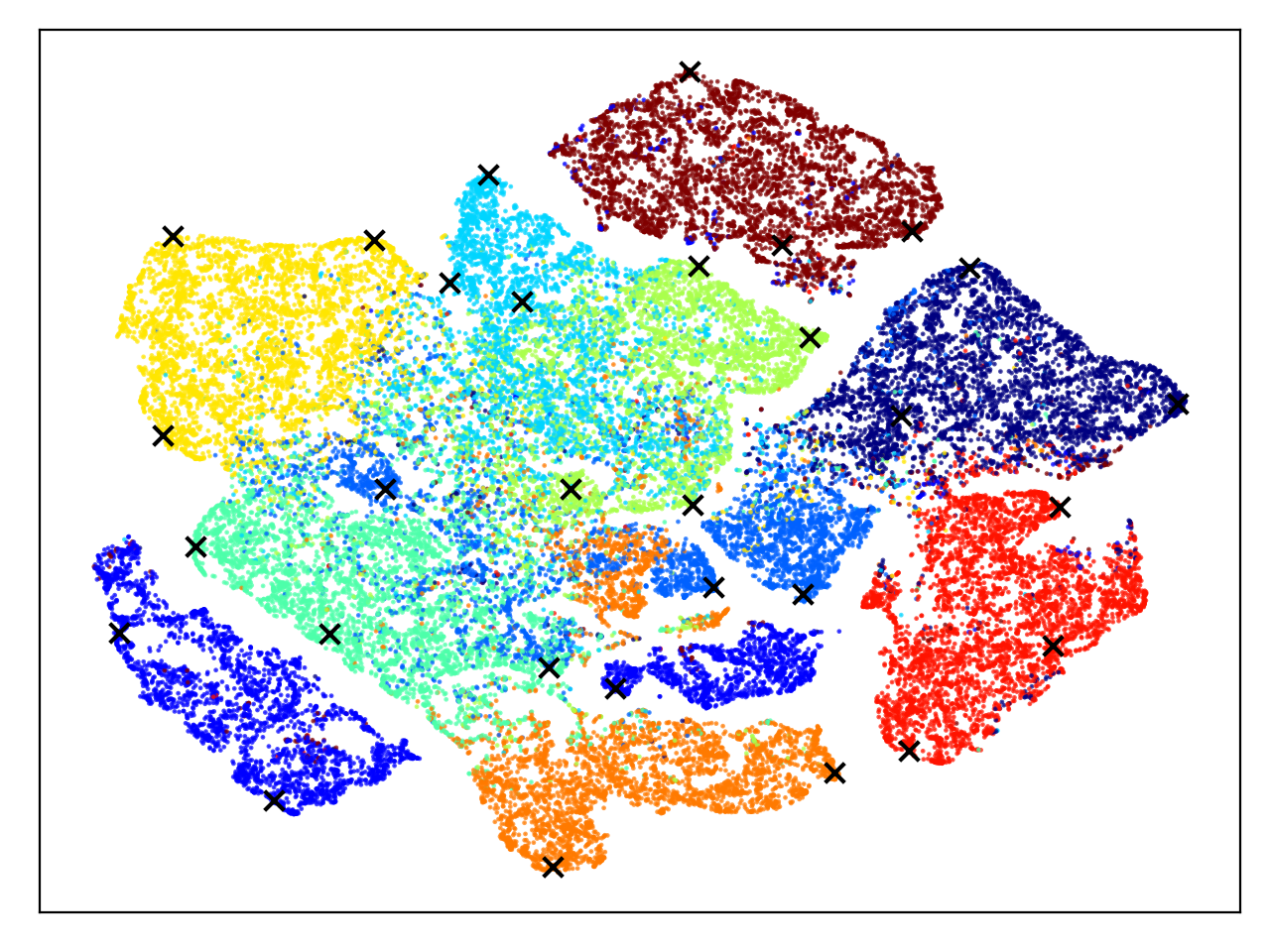}
      \caption{Color by labels}
      \label{subfig:tsne_labels}
    \end{subfigure}
    \begin{subfigure}{.27\textwidth}
      \centering
      \includegraphics[width=1\linewidth]{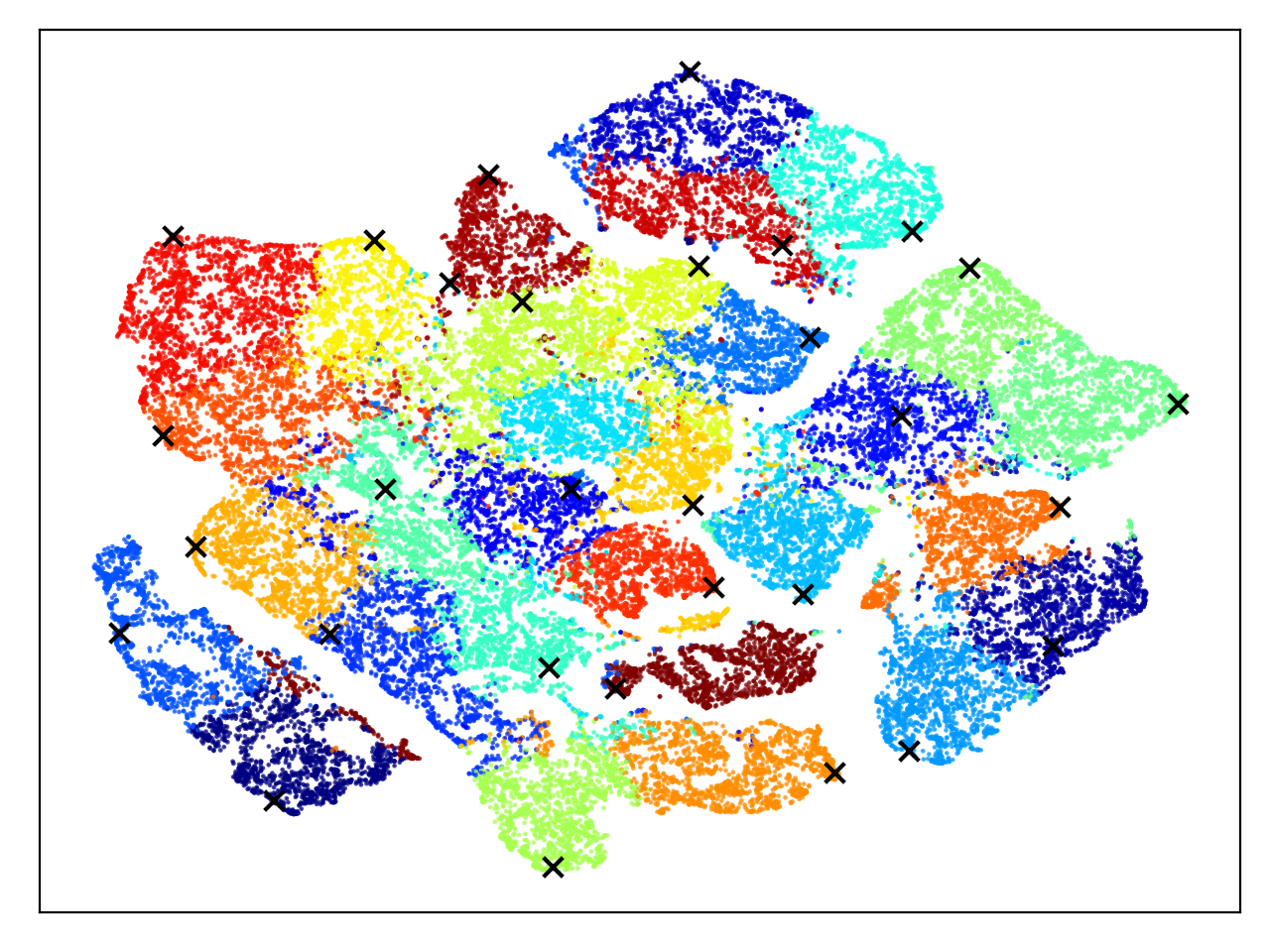}
      \caption{Colored by cluster assignment}
      \label{subfig:tsne_clusters}
    \end{subfigure}
    \begin{subfigure}{.27\textwidth}
      \centering
      \includegraphics[width=1\linewidth]{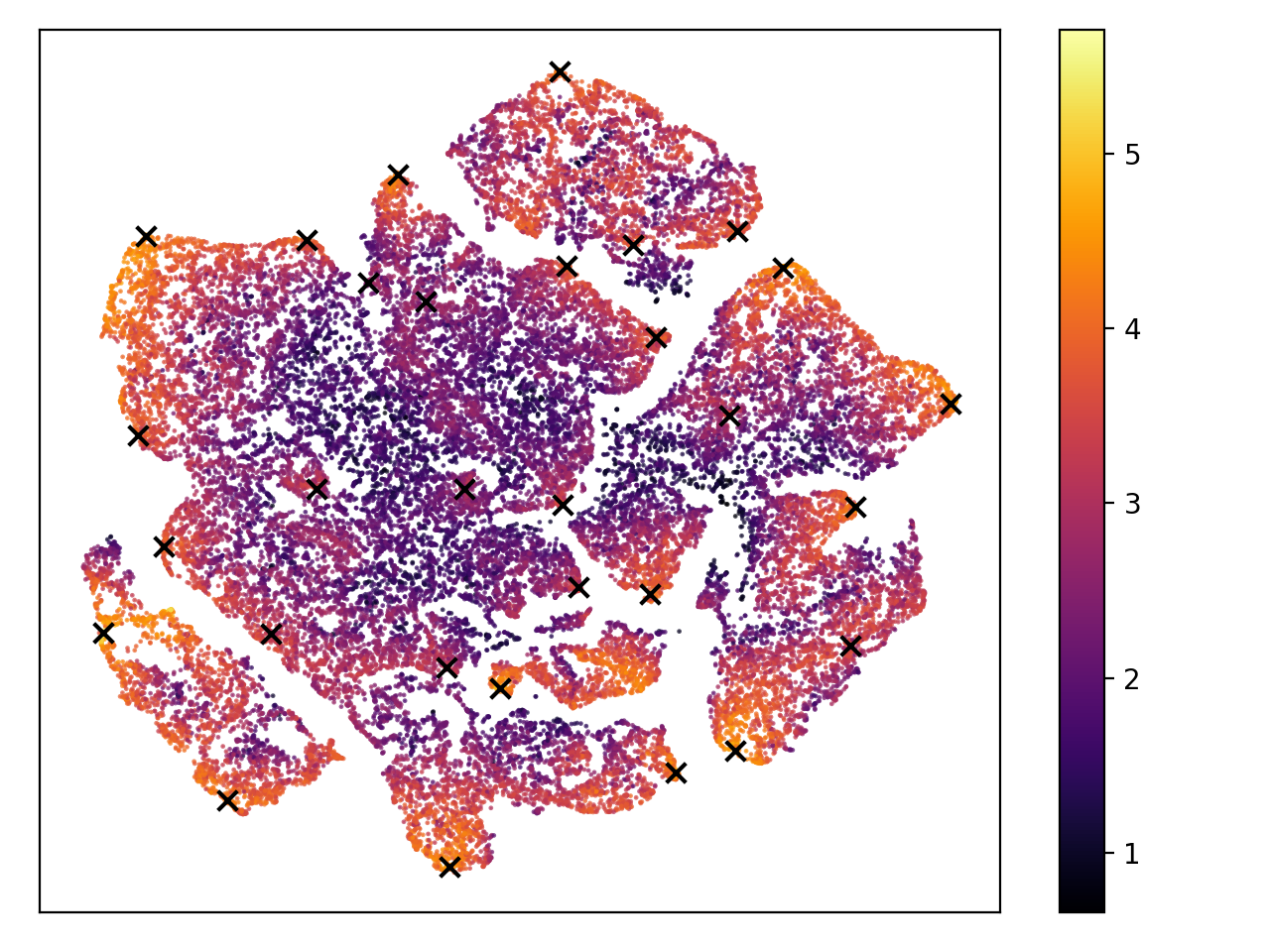}
      \caption{Color by log-density}
      \label{subfig:tsne_density}
    \end{subfigure}
    \begin{center}
    \begin{subfigure}{.55\textwidth}
      \centering
      \includegraphics[width=1\linewidth]{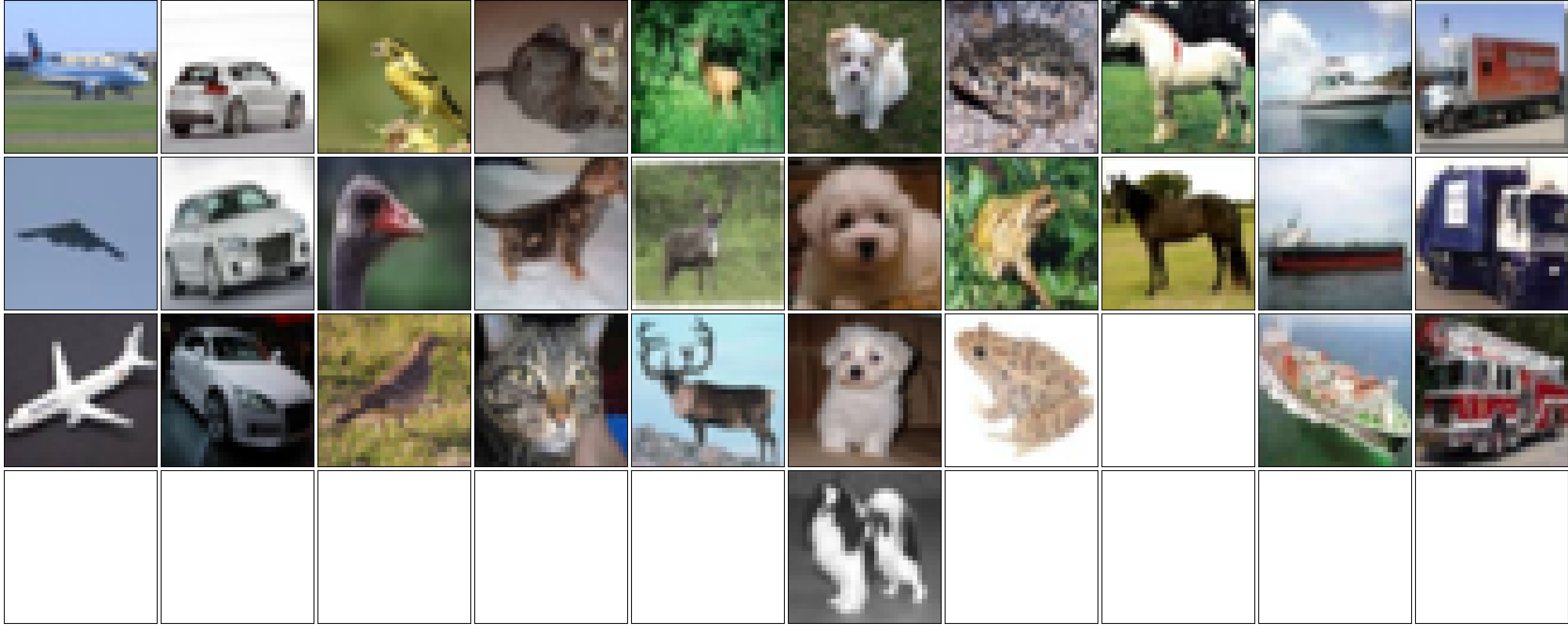}
      \caption{The $30$ samples are selected after clustering by SCAN to $30$ clusters.}
      \label{subfig:30_images_scan}
    \end{subfigure}
    \end{center}
    \caption{(a)-(c) Visualizing the selection of $30$ examples using the SCAN clustering algorithm -- examples marked with $\boldsymbol{\times}$ are selected for labeling. (d) The selected images, each column represents a different label.}
    \label{fig:tsne}
\end{center}
\vspace{-0.6cm}
\end{figure*}

\begin{figure}[thb!]
\begin{center}
      \includegraphics[width=.48\textwidth]{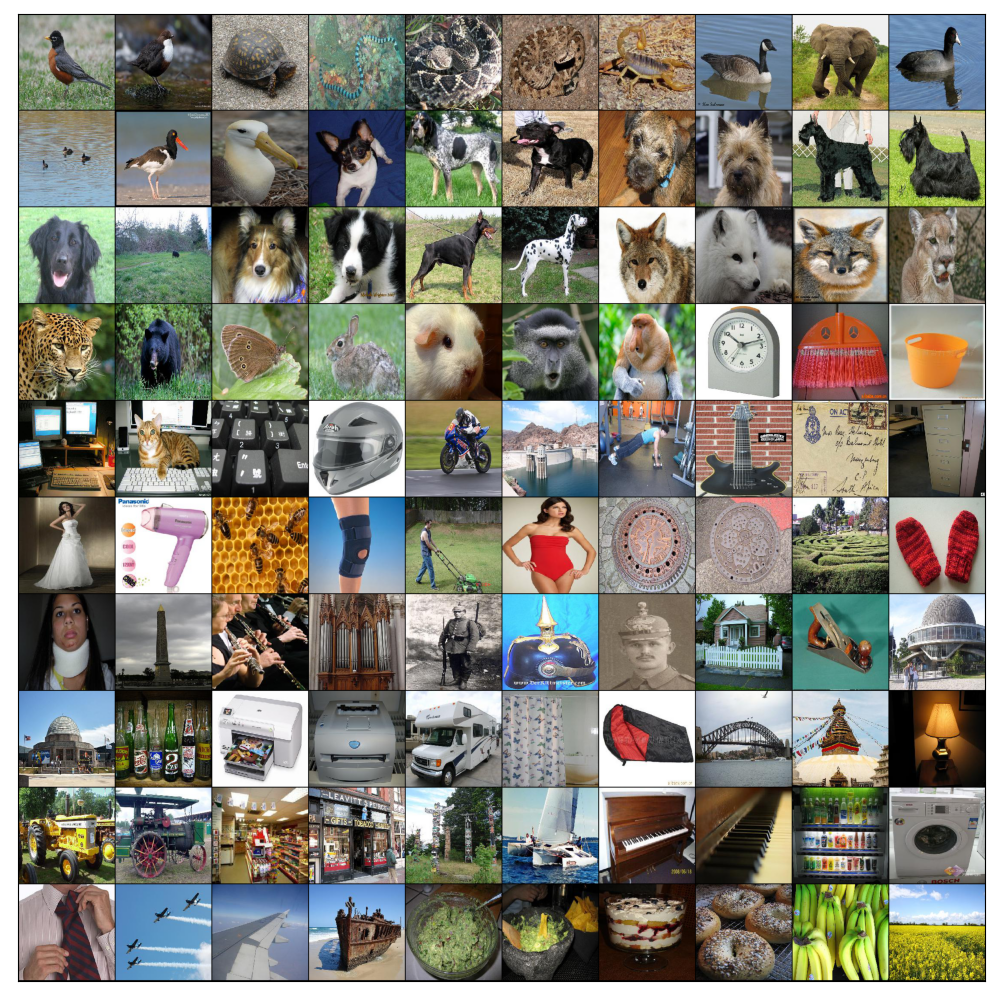}
    \caption{100 ImageNet-100 examples selected by \emph{TypiClust}. }
    \label{fig:imagenet_images}
\end{center}
\end{figure}

\subsection{Error Scores of Deep Neural Networks}
\label{sec:error-dnn}

Next, we plot the error scores of deep neural networks on image classification tasks. In all the datasets we evaluated, the error of deep networks as a function of the number of examples drops much faster than an exponential function and therefore can be shown to be undulating. In practice, such error functions are bounded from above by an exponent, and hence are also SP-undulating. To see some examples of error functions of neural networks trained on super-classes of CIFAR-100, refer to Fig.~\ref{fig:error_scores_neural_networks}.

\section{Visualization of Query Selection}


    


Fig.~\ref{fig:tsne} demonstrates the selection of $30$ examples from CIFAR-10 using \emph{TypiClust} in greater detail. Recall that TypiClust first clusters the dataset to $30$ clusters -- using SCAN clustering algorithm. We plot the tSNE dimensionality reduction of the model's feature space, colored in various ways: Fig.~\ref{subfig:tsne_labels} shows the tSNE embedding colored by the GT labels.  Fig.~\ref{subfig:tsne_clusters} shows the tSNE embedding colored by the cluster assignment.  Fig.~\ref{subfig:tsne_density} shows the tSNE embedding colored by the log density (for better visualization).  Examples marked with $\boldsymbol{\times}$ are selected for labeling. Fig.~\ref{subfig:30_images_scan} shows the selected images. Fig.~\ref{fig:imagenet_images} shows 100 examples selected by \emph{TypiClust} from ImageNet-100.

\section{Implementation Details}
\subsection{Method Implementation Details}
\label{app:method_implementation_details}
\myparagraph{Step 1: Representation learning -- CIFAR and TinyImageNet.}
We trained SimCLR using the code provided by \citet{van2020scan} for CIFAR-10, CIFAR-100 and TinyImageNet. Specifically, we used ResNet18 with an MLP projection layer to a $128$ vector, trained for 500 epochs. All the training hyper-parameters were identical to those used by SCAN.
After training, we used the $512$ dimensional penultimate layer as the representation space.
As in SCAN, we used an SGD optimizer with 0.9 momentum, and an initial learning rate of 0.4 with a cosine scheduler. The batch size was 512 and weight decay of 0.0001.
The augmentations were random resized crops, random horizontal flips, color jittering, and random grayscaling. We refer to \citet{van2020scan} for additional details. We used the L2 normalized penultimate layer as embedding.

\myparagraph{Step 1: Representation learning -- ImageNet.}
We extracted embedding from the official (ViT-S/16) DINO weights pre-trained on ImageNet. We used the L2 normalized penultimate layer as embedding.

\myparagraph{Step 2: Clustering for diversity.}
We limited the number of clusters when partitioning the data to  $max\_clusters$ (a hyperparameter). This parameter was arbitrarily picked as $500$ for CIFAR-10 and CIFAR-100 and $1000$ for TinyImageNet and ImageNet subsets (other values resulted in similar behavior).
This was done for two reasons: \begin{inparaenum}[(a)]
    \item prevent over clustering (ending up with clusters that are too small);
    \item stabilize the clustering algorithms.
\end{inparaenum}
The number of clusters chosen is $K=\min(|L_{i-1}|+B, max\_clusters)$.

\myparagraph{K-Means.} 
We used scikit-learn KMeans when $K\le 50$ and MiniBatchKMeans otherwise. This was done to reduce runtime when the number of clusters is large.

\myparagraph{SCAN.} 
We used the code provided by SCAN and modified the number of clusters to $K$. We only trained the first step of SCAN (we did not perform the pseudo labeling step, since it degraded clustering performance).

\myparagraph{Step 3: Clustering for diversity.}
Since we introduced $max\_cluster$, we are no longer guaranteed to have $B$ clusters that don't intersect the labeled set.
Moreover, to estimate typicality, we require $>20$ samples in every cluster. To solve this, we used $\min\{20, cluster\_size\}$ nearest neighbors. To avoid inaccurate estimation of the typicality, we dropped clusters with less than $5$ samples\footnote{This limiting case was rarely encountered, as clusters are usually balanced.}.

After all is done, the method adds points iteratively until the budget is exhausted, in the following way: 
\begin{inparaenum}[(1)]
    \item Out of the clusters with the fewest labeled points and of size larger than 5, select the largest cluster.
    \item Compute the Typicality of every point in the selected cluster, using $\min\{20, cluster\_size\}$ neighbors.
    \item Add to the query the point with the highest typicality.
\end{inparaenum}

\subsection{Evaluation and Implementation Details}
\label{app:eval_impl_details}

\subsubsection{Fully Supervised Evaluation}
We used the active learning comparison framework by \citet{Munjal2020TowardsRA}. 
Specifically, we trained a ResNet18 on the labeled set, optimizing using SGD with $0.9$ momentum and Nesterov momentum. The initial learning rate is $0.025$ and was modified using a cosine scheduler. The augmentations used are random crops and horizontal flips. Our changes to this framework are listed below.
\paragraph{Re-Initialize weights between iterations}
When training with extremely low budgets, networks tend to produce over-confident predictions. As a result, when querying samples and fine-tuning from the existing model, the loss tends to ``spike'', which leads to optimization issues. Therefore, we re-initialized the weights between iterations.

\paragraph{TinyImageNet modifications}
As training did not converge in the original implementation over Tiny-ImageNet, we increased the number of epochs from $100$ to $200$ and changed the minimal crop side from $0.08$ to $0.5$. This ensured stable  convergence and a more reliable evaluation in AL experiments.

\paragraph{ImageNet modifications}
ImageNet hyper-parameters were identical to TinyImageNet except for the number of epochs, which was set to $100$ due to high computational cost, and the batch size, which was set to $50$ to fit into a standard GPU virtual memory.

\subsubsection{Linear Evaluation on Self-Supervised Embedding}
\label{app:linear_eval_implementation}
In these experiments, we also used the framework by \citet{Munjal2020TowardsRA}.
We extracted an embedding as described in Appendix ~\ref{app:method_implementation_details}, and trained a single linear layer of size $d \times C$ where $d$ is the feature dimensions, and $C$ is the number of classes. To optimize this single layer, we increased the initial learning rate by a factor of 100 to $2.5$, and as the training time is much shorter, we multiplied the number of epochs by $2$.

\subsubsection{Semi-Supervised Evaluation}
\label{app:semi_implementation}
When training FlexMatch, we used the semi-supervised framework by \citet{DBLP:journals/corr/abs-2110-08263}. All experiments were repeated 3 times.
We used the following hyper-parameters when training each experiment:

\myparagraph{CIFAR-10.} We trained WideResNet-28, for 400k iterations. We used SGD optimizer, with $0.03$ learning rate, $64$ batch size, $0.9$ momentum, $0.0005$ weight decay, $2$ widen factor, $0.1$ leaky slope and without dropout. 
The augmentations are similar to those used in FlexMatch. The weak augmentations include random crops and horizontal flips, while the strong augmentations are according to RandAugment \citep{cubuk2020randaugment}.

\myparagraph{CIFAR-100.} We trained WideResNet-28, for 400k iterations. We used an SGD optimizer, with $0.03$ learning rate, $64$ batch size, $0.9$ momentum, $0.0001$ weight decay, $8$ widen factor, $0.1$ leaky slope, and without dropout. The augmentations are similar to those used in CIFAR-10.

\myparagraph{TinyImageNet.} We trained ResNet-50, for 1.1m iterations. We used an SGD optimizer, with a $0.03$ learning rate, $32$ batch size, $0.9$ momentum, $0.0003$ weight decay, $0.1$ leaky slope, and without dropout. The augmentations are similar to those used in FlexMatch.

\subsection{Ablation studies}
\label{app:sec:ablation_details}
\myparagraph{Margin by an ``oracle'' network.} In Section~\ref{sec:ablation_oracle_uncertainty}, we compute an ``oracle'' uncertainty measure. When training an oracle, we use a VGG-19 \citep{simonyan2014very} trained on CIFAR-10, using the hyper-parameters of the original paper. We calculate the margin of each example according to this network. We note that an AL strategy based on this margin works well in the high-budget regime for both the oracle and the ``student'' network.

\section{Additional Empirical Results}
\label{app:more_empirical}

\subsection{Supervised Framework}
\label{app:more_sup_empirical}
In the main paper, we presented results on 1 and 5 samples per class on average. Fig.~\ref{fig:app_cifar_more_budgets} shows similar results using additional budget sizes.

\begin{figure}[thb!]
\begin{center}
\begin{subfigure}{.45\textwidth}
  \centering
 \includegraphics[width=\linewidth]{avihu_graphs/maybe_camera_ready/badge_legend.png}
\end{subfigure}
\\
    \begin{subfigure}{.157\textwidth}
      \centering
      \includegraphics[width=\linewidth]{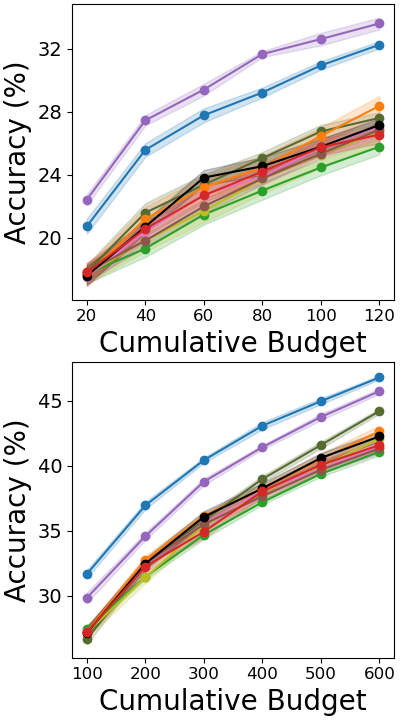}
    \caption{CIFAR-10}
    \label{fig:app_al_graph_cifar10}
    \end{subfigure}
    \begin{subfigure}{.157\textwidth}
      \centering
      \includegraphics[width=\linewidth]{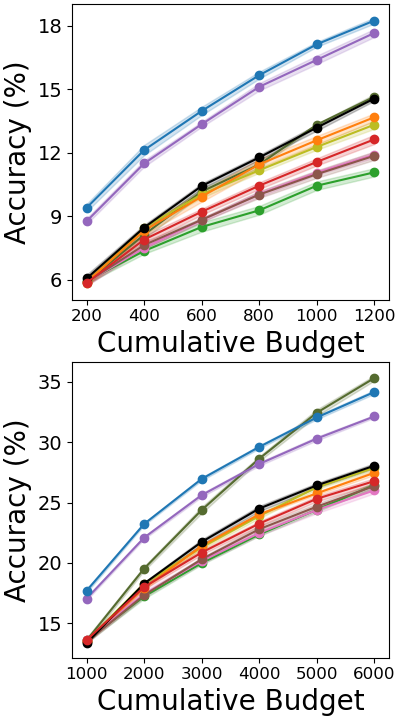}
    \caption{CIFAR-100}
    \label{fig:app_al_graph_cifar100}
    \end{subfigure}
    
\caption{Additional results in the supervised framework, including an average of 2 and 10 examples per class on CIFAR10 and CIFAR100, similarly to Fig.~\ref{fig:main_al_graph}.}
\label{fig:app_cifar_more_budgets}
\end{center}
\end{figure}

In Fig.~\ref{fig:app_more_imagenet} we present results on additional datasets, which include ImageNet-50, ImageNet-100 and TinyImageNet.
\emph{TypiClust} outperforms all competing methods on these datasets as well.

\begin{figure}[thb!]
\begin{center}
\begin{subfigure}{.45\textwidth}
  \centering
 \includegraphics[width=\linewidth]{avihu_graphs/maybe_camera_ready/badge_legend.png}
\end{subfigure}
\\
    \begin{subfigure}{.157\textwidth}
      \centering
      \includegraphics[width=\linewidth]{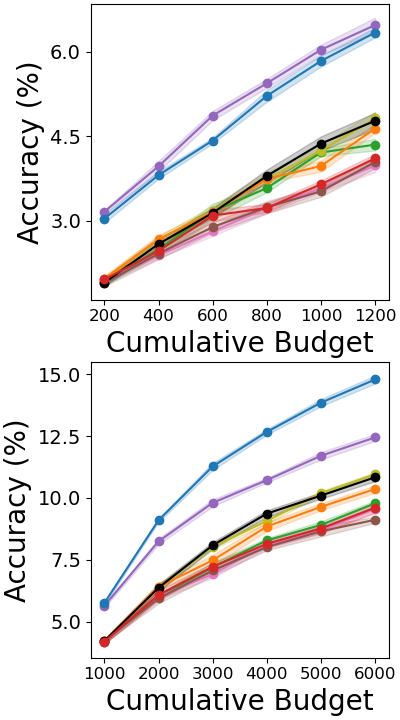}
    \caption{TinyImageNet}
    \label{fig:app_more_tinyimagenet}
    \end{subfigure}
    \begin{subfigure}{.157\textwidth}
      \centering
      \includegraphics[width=\linewidth]{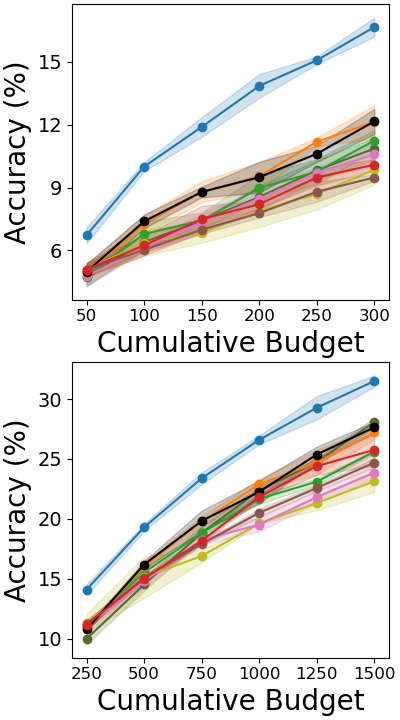}
    \caption{ImageNet-50}
    \label{fig:app_more_imagenet_50}
    \end{subfigure}
    \begin{subfigure}{.157\textwidth}
      \centering
      \includegraphics[width=\linewidth]{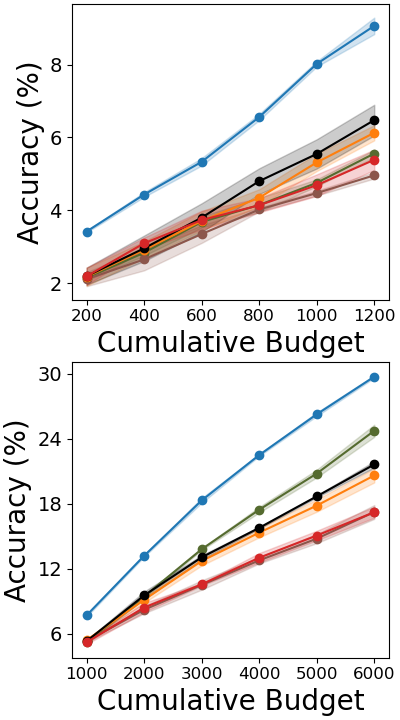}
    \caption{ImageNet-200}
    \label{fig:app_more_imagenet_100}
    \end{subfigure}
\caption{Similar to Fig.~\ref{fig:main_al_graph}, we report the results on TinyImageNet, ImageNet-50 and ImageNet-200.
}
\label{fig:app_more_imagenet}
\end{center}
\vspace{-.6cm}
\end{figure}

\subsubsection{Random Initial pool}
In Section~\ref{sec:starting_from_random}, we show that even when the initial pool is sampled randomly, \emph{TypiClust} still improves over random selection. Fig.~\ref{fig:app:random_init_pool} provides additional evidence for this phenomenon, on several datasets and budget sizes.

\begin{figure}[htb!]
\begin{subfigure}{.45\textwidth}
  \centering
 \includegraphics[width=\linewidth]{avihu_graphs/maybe_camera_ready/badge_legend.png}
\end{subfigure}
\\
\begin{center}
    \begin{subfigure}{.157\textwidth}
      \centering
      \includegraphics[width=\linewidth]{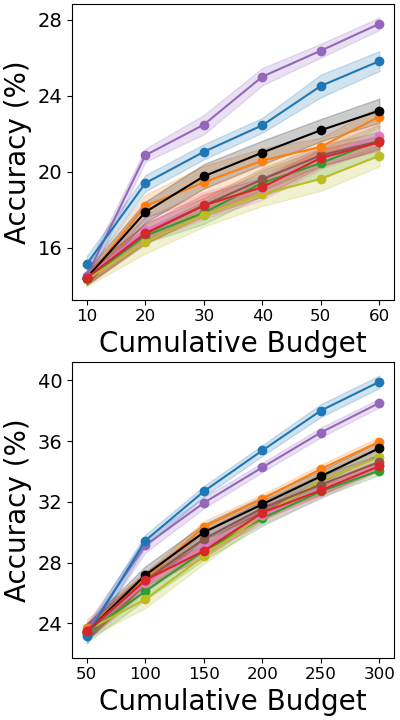}
    \caption{CIFAR-10}
    \end{subfigure}
    \begin{subfigure}{.157\textwidth}
      \centering
      \includegraphics[width=\linewidth]{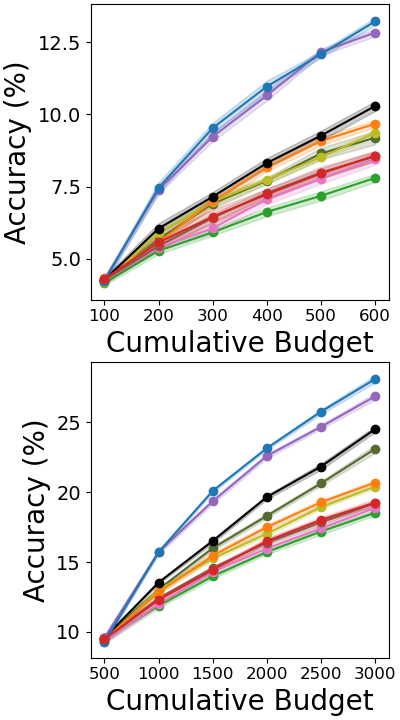}
    \caption{CIFAR-100}
    \end{subfigure}
    \begin{subfigure}{.157\textwidth}
      \centering
      \includegraphics[width=\linewidth]{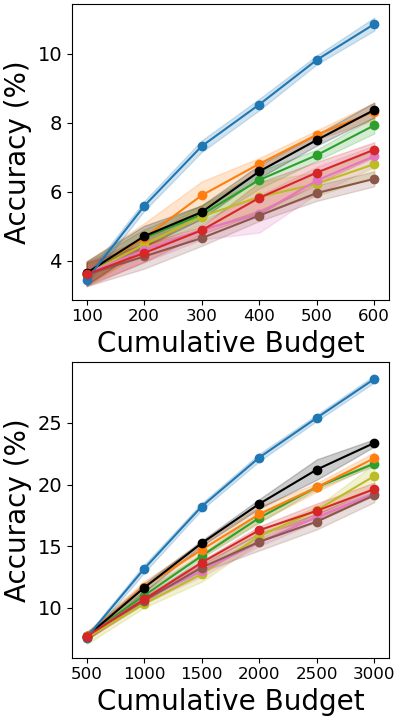}
    \caption{ImageNet-100}
    \end{subfigure}
\caption{Similar to Fig. ~\ref{fig:cifar10_random_init}, we also report results on CIFAR-100 and ImageNet-100 with an additional budget.}
\vspace{-0.3cm}
\label{fig:app:random_init_pool}
\end{center}
\end{figure}

\subsubsection{Imbalanced CIFAR-10}
Fig.~\ref{fig:app:imbalanced_data} presents results on an imbalanced subset of CIFAR-10, where the number of samples per class decreases exponentially.

\begin{figure}[htb!]
\begin{center}
\begin{subfigure}{.45\textwidth}
  \centering
 \includegraphics[width=\linewidth]{avihu_graphs/maybe_camera_ready/badge_legend.png}
\end{subfigure}
\\
    \begin{subfigure}{.45\textwidth}
      \centering
      \includegraphics[width=\linewidth]{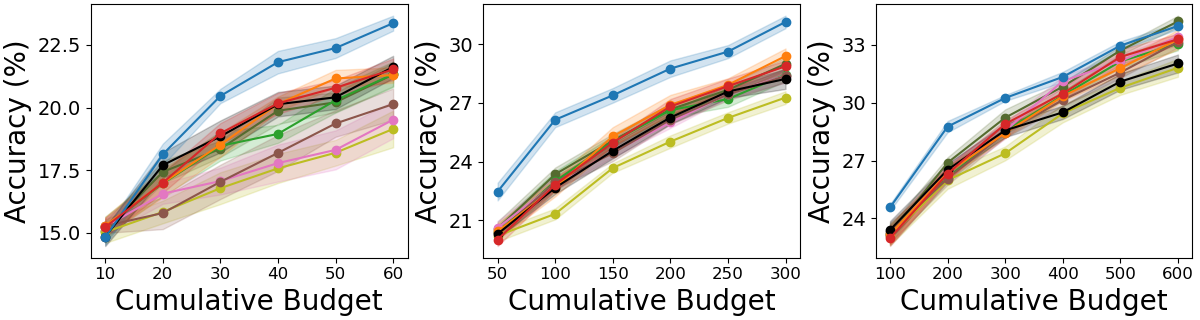}
    \end{subfigure}
\caption{Imbalanced CIFAR-10 results on $3$ different budgets.}
\vspace{-0.3cm}
\label{fig:app:imbalanced_data}
\end{center}
\end{figure}

\subsection{Supervised Using Self-Supervised Embeddings}

In Fig.~\ref{fig:app_linear_more}, we show results using additional datasets of a linear classifier trained over a self-supervised pre-trained embedding. We see that the initial pool selection provides a very large boost in performance -- especially with ImageNet-50 and ImageNet-200.  

\label{app:more_lin_empirical}
\begin{figure}[thb!]
\begin{center}
\begin{subfigure}{.45\textwidth}
  \centering
 \includegraphics[width=\linewidth]{avihu_graphs/maybe_camera_ready/badge_legend.png}
\end{subfigure}
\\
    \begin{subfigure}{.157\textwidth}
      \centering
      \includegraphics[width=\linewidth]{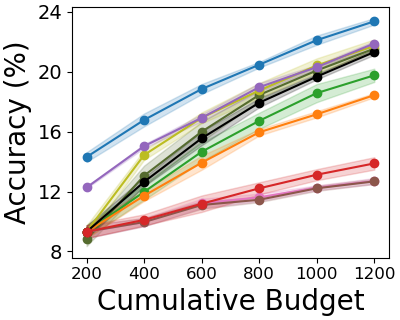}
    \caption{TinyImageNet}
    \label{fig:app_linear_more_tinyimagenet}
    \end{subfigure}
    \begin{subfigure}{.157\textwidth}
      \centering
      \includegraphics[width=\linewidth]{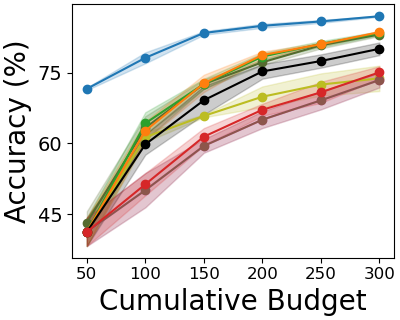}
    \caption{ImageNet-50}
    \label{fig:app_linear_more_imagenet50}
    \end{subfigure}
    \begin{subfigure}{.157\textwidth}
      \centering
      \includegraphics[width=\linewidth]{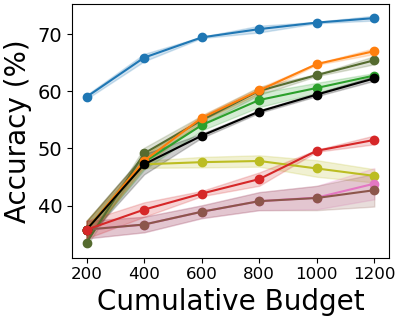}
    \caption{ImageNet-200}
    \label{fig:app_linear_more_imagenet100}
    \end{subfigure}
\caption{Similar to Fig.~\ref{fig:weak_semi}, we report the linear evaluation results on TinyImageNet, ImageNet-50 and ImageNet-200.}
\label{fig:app_linear_more}
\end{center}
\vspace{-.6cm}
\end{figure}

\subsection{Semi-Supervised Framework}
\label{app:more_ssl_empirical}

Finally, below we describe additional experiments to those plotted in Fig.~\ref{fig:semi_supervised}, within the semi-supervised framework.

To test the dependency of our deep clustering variant of \emph{TypiClust} on SCAN, we evaluated another variant based on RUC \citep{park2021improving}, which is henceforth denoted $TPC_{RUC}$. We plot its performance on CIFAR-10 and CIFAR-100 in Fig.~\ref{fig:app:semi_supervised}. As RUC is computationally demanding, we fix the number of clusters to the number of classes in the corresponding dataset, and then (when needed) further sub-cluster the data using K-means to the desired number of clusters. In all the tested settings, $TPC_{RUC}$ surpassed the performance of the random baseline by a large margin, suggesting that using SCAN is not crucial for \emph{TypiClust}, and may be replaced by any competitive clustering algorithm.

\begin{figure}[htb!]
\begin{center}
    \begin{subfigure}{.2\textwidth}
      \centering
      \includegraphics[width=\linewidth]{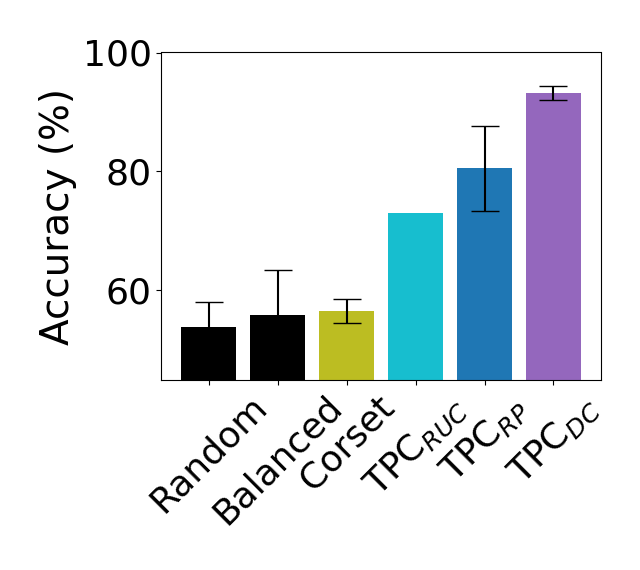}
    \vspace{-0.75cm}
    \caption{CIFAR-10 40 labels}
    \label{fig:app:ssl_cifar_10_with_40_examples}
    \end{subfigure}
    \begin{subfigure}{.2\textwidth}
      \centering
      \includegraphics[width=\linewidth]{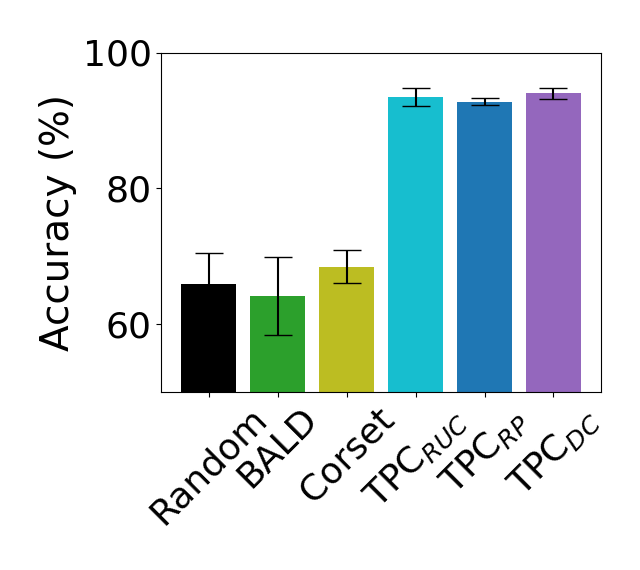}
    \vspace{-0.75cm}
    \caption{CIFAR-10 10 labels}
    \end{subfigure}
    \begin{subfigure}{.2\textwidth}
      \centering
      \includegraphics[width=\linewidth]{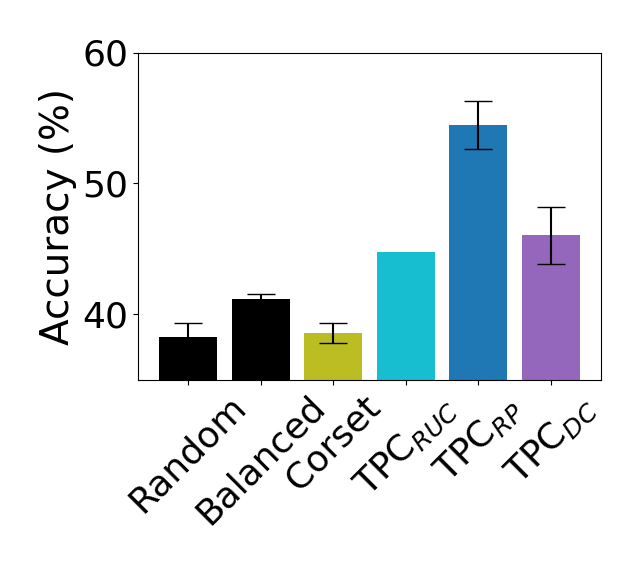}
    \vspace{-0.75cm}
    \caption{CIFAR 100}
    \end{subfigure}
\caption{Similar to Fig.~\ref{fig:semi_supervised}, using the $TPC_{RUC}$ variant of \emph{TypiClust}. (a) $40$ labels on CIFAR-10, (b) $10$ labels on CIFAR-10, (c) $300$ labels on CIFAR-100}
\label{fig:app:semi_supervised}
\vspace{-0.3cm}
\end{center}
\end{figure}

\begin{figure}[htb!]
\begin{center}
    \begin{subfigure}{.25\textwidth}
      \centering
      \includegraphics[width=\linewidth]{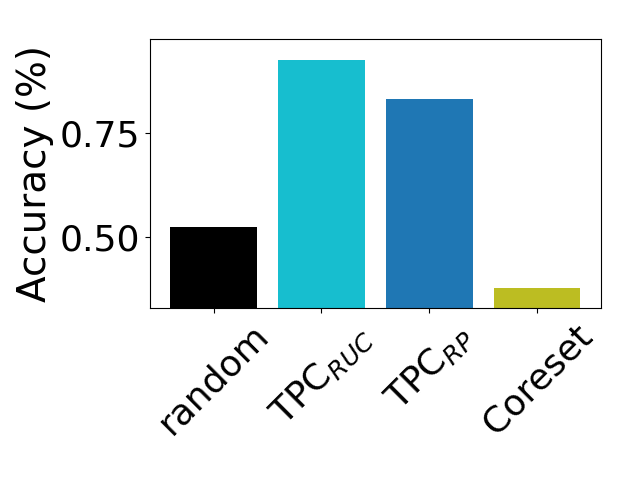}
    \vspace{-0.75cm}
    \caption{CIFAR-10 20 labels}
    \end{subfigure}
\caption{Similar to Fig.~\ref{fig:semi_supervised}, using Semi-MMDC instead of FlexMatch. We train $20$ labels on CIFAR-10, observing a large performance gain when using \emph{TypiClust} to perform the initial selection of the labeled data.}
\vspace{-0.3cm}
\label{fig:app:boaz_cifar10_20_labels}
\end{center}
\end{figure}

Additionally, we performed experiments with other budgets. In Fig.~\ref{fig:app:ssl_cifar_10_with_40_examples}, we plot the same experiment as Fig.~\ref{fig:ssl_cifar_10_with_10_examples}, with a budget of $40$ examples on CIFAR-10, seeing similar results.

To verify that the observed performance boost is not unique to FlexMatch, we repeat the same experiments with another competitive semi-supervised learning method, 	
Semi-MMDC \citep{lerner2020boosting}. Using the code provided by \citet{lerner2020boosting}, and following the exact training protocol, we train Semi-MMDC using $20$ labels on CIFAR-10. Similarly to the results on FlexMatch, we report a significant increase in performance when training on examples chosen by \emph{TypiClust}, see Fig.~\ref{fig:app:boaz_cifar10_20_labels}.

We note that in all the experiments we performed within the low budget regime, \emph{TypiClust} always surpassed the random baseline by a large margin.

\end{document}